\documentclass[twoside]{article}

\usepackage[preprint]{aistats2026}
\usepackage[ruled,vlined]{algorithm2e}  
\usepackage{bm}  
\usepackage{amsfonts,amsmath,amssymb,amsthm}
\usepackage{graphicx}
\usepackage{natbib}
\usepackage{hyperref}
\usepackage{pifont}

\usepackage{xcolor}
\usepackage{subcaption}
\usepackage{hyperref}
\hypersetup{
    colorlinks,
    linkcolor={black},
    citecolor={black},
    urlcolor={black}
}
\usepackage{soul}
\usepackage{xspace}
\usepackage{booktabs} 
\newtheorem{theorem}{Theorem}[section]

\newtheorem{proposition}[theorem]{Proposition}

\newcommand{\piot}{\pi^{\texttt{OT}}}

\newcommand{\setD}{\mathcal{D}}

\newcommand{\setH}{\mathcal{H}}
\newcommand{\setX}{\mathcal{X}}
\newcommand{\setY}{\mathcal{Y}}

\newcommand{\hatmu}{\hat{\mu}}

\newcommand{\inner}[2]{\left\langle #1,#2 \right\rangle}

\begin{document}

%

%

\twocolumn[

\aistatstitle{Coupled Flow Matching}

\aistatsauthor{ Wenxi Cai ${}^1$ \And Yuheng Wang ${}^1$ \And  Naichen Shi ${}^2$}

\aistatsaddress{ ${}^1$ University of Michigan \And ${}^2$Northwestern University} ]

\begin{abstract}
We introduce Coupled Flow Matching (CPFM), a framework that integrates controllable dimensionality reduction and high-fidelity reconstruction. CPFM learns coupled continuous flows for both the high-dimensional data $x$ and the low-dimensional embedding $y$, which enables sampling $p(y|x)$ via a latent-space flow and $p(x|y)$ via a data-space flow. Unlike classical dimension-reduction methods, where information discarded during compression is often difficult to recover, CPFM preserves the knowledge of residual information within the weights of a flow network. This design provides a bespoke controllability: users may decide which semantic factors to retain explicitly in the latent space, while the complementary information remains recoverable through the flow network. Coupled flow matching builds on two components: (i) an extended Gromov–Wasserstein optimal transport objective that establishes a probabilistic correspondence between data and embeddings, and (ii) a dual-conditional flow-matching network that extrapolates the correspondence to the underlying space. Experiments on multiple benchmarks show that CPFM yields semantically rich embeddings and reconstructs data with higher fidelity than existing baselines.

\end{abstract}

\section{Introduction}
Dimension reduction has long been a central tool in statistical learning for extracting informative low-dimensional embeddings $y$ from high-dimensional data $x$. A classical example is Principal Component Analysis (PCA)~\citep{shlens2014tutorialprincipalcomponentanalysis}, which efficiently captures dominant linear structures. Nonlinear approaches such as Laplacian Eigenmaps~\citep{10.1162/089976603321780317}, t-SNE~\citep{JMLR:v9:vandermaaten08a}, and UMAP~\citep{mcinnes2020umapuniformmanifoldapproximation} extend this principle to nonlinear mappings, with an emphasis on preserving local neighborhoods. These embeddings $y$ have proven effective for visualization, denoising, and downstream analysis, and are now indispensable across many scientific domains\citep{Kobak2019Art,Armstrong2021UMAP,Becht2019UMAP}. 

Nevertheless, dimension reduction inevitably loses information, as the mapping from high-dimensional $x$ to low-dimensional $y$ is non-invertible by nature. As a result, the global structure is hard to preserve. 
Furthermore, reconstructing $x$ from $y$ becomes challenging. 

Recent advances in generative models have sought to employ deep neural networks to establish statistical connections between $x$ and $y$. Autoencoders~\citep{doi:10.1126/science.1127647} and VAEs~\citep{kingma2022autoencodingvariationalbayes} enable simultaneous data compression and reconstruction using encoder and decoder neural networks. Generative Adversarial Networks (GANs)~\citep{goodfellow2014generativeadversarialnetworks}, such as Style-GAN~\citep{stylegan}, further demonstrate the power of learned latent spaces in generating high-fidelity samples from compact codes. Despite their success, these approaches still face two critical challenges: (i) the latent
representation $y$ often entangles multiple explanatory factors in $x$, making it difficult to interpret or disentangle task-relevant information, and (ii) the geometry
of the latent space is typically an implicit byproduct of training rather than an explicitly controlled structure.

In light of such limitations, we propose \textbf{controllability} as a design principle. We advocate a controllable dimensionality reduction paradigm that makes the choices of important information explicit. Specifically, the embedding $y$ should: (i) highlight user-specified or task-relevant statistics, (ii) enforce latent geometry through constraints or priors, and (iii) neglect nuisance variation while retaining the ability to recover it. This controllability principle yields three key benefits: targeted submanifold design guided by prior knowledge, improved interpretability, and informed generation of new data $x$.

To implement this principle, we propose \textbf{Coupled Flow Matching (CPFM)}, a controllable dimensionality reduction framework that learns bidirectional flows between high-dimensional data $x$ and low-dimensional embeddings $y$. CPFM operates in two stages.  

In the \textit{first stage}, we construct a probabilistic coupling between discrete samples of $x$ and $y$ via a generalized Gromov–Wasserstein optimal transport (GWOT). Our formulation incorporates a kernel function on the $x$-space to encode relational structure, semantic labels, or other user-specified priors. This kernel induces a transport cost that quantifies the distortion between relations in the original $x$-space and those in the embedding $y$-space, thereby allowing explicit control over the coupling. The resulting optimal transport plan $\pi^{\mathrm{OT}}$ is thus guided directly by user-defined knowledge. To overcome the computational hardness of generic OT, we design an alternating minimization algorithm for an entropy-regularized version of GWOT whose per-iteration complexity is $\mathcal{O}(n^2)$, where $n$ is the dataset size.  

In the \textit{second stage}, we extrapolate $\pi^{\mathrm{OT}}$ beyond the training samples to define conditional distributions $p(y|x)$ and $p(x|y)$ via DCMF. While classical flow matching transports samples between two fixed distributions, conditional variants leverage side information to learn flow fields. We extend this paradigm by introducing a dual conditional mechanism within a shared drift network: conditioning on $x$, the network predicts flows on $y$ to generate $p(y|x)$; conditioning on $y$, it predicts flows on $x$ to generate $p(x|y)$. During training, we develop a mute-masking strategy to ensure that only the
active direction contributes to the loss. This design leverages
the symmetry between the two conditional flows and
avoids redundant modeling.

We summarize our main contributions:  
\vspace{-0.2cm}  
\begin{itemize}  
\item \textbf{Generalized GWOT.} We propose a new kernelized quadratic optimal transport objective that embeds semantic priors and enables controllable alignment between data $x$ and embeddings $y$.  

\item \textbf{Efficient OT solver.} We develop a new alternating optimization algorithm for entropy-regularized GWOT with per-iteration complexity $\mathcal{O}(n^2)$, scalable to large datasets.  

\item \textbf{Dual conditional flow matching.} We design a new shared drift network with dual conditioning to jointly learn $p(y|x)$ and $p(x|y)$.  

\item \textbf{Comprehensive evaluation.} We implement CPFM on a range of dimensionality reduction and reconstruction benchmarks, highlighting superior performance in both embedding quality and generative fidelity.  
\end{itemize}  

\section{Related Work}
We review several lines of research that are most relevant to our framework.

\paragraph{Flow matching.}  
Flow Matching learns continuous-time velocity fields and has recently scaled to large-scale image generation~\citep{lipman2023flowmatchinggenerativemodeling}. Several extensions improve efficiency and flexibility: Rectified Flow straightens transport paths for accelerated sampling~\citep{liu2022flowstraightfastlearning}, with a variational extension introducing a variational objective into this framework~\citep{guo2025variationalrectifiedflowmatching}. Stochastic Interpolants unify diffusion and deterministic flows through interpolation-based objectives with likelihood control~\citep{albergo2023stochasticinterpolantsunifyingframework}. On the algorithmic side, Weighted CFM reduces variance via importance weighting~\citep{calvoordonez2025weightedconditionalflowmatching}, Context-Varying CFM adapts conditioning to changing contexts~\citep{generale2025conditionalvariableflowmatching}, and Optimal Flow Matching directly optimizes straight-line paths~\citep{tong2024improvinggeneralizingflowbasedgenerative}. Despite these advances, standard flow-matching models construct probability flows in fixed-dimensional spaces and do not provide a mechanism for dimensionality reduction.  

\paragraph{Cross-modal generative modeling.}  
Multimodal diffusion and flow-based models extend generation to joint distributions across modalities. For example, UniDiffuser fits multimodal joint distributions with a unified model~\citep{bao2023transformerfitsdistributionsmultimodal}, EasyGen~\citep{zhao2024easygeneasingmultimodalgeneration} combines bi-directional conditional diffusion with large language models for interactive cross-modal tasks, and Dual Diffusion Transformer~\citep{li2025dualdiffusionunifiedimage} provides a backbone for joint text–image modeling. While powerful, these models aim to capture cross-modal correspondences rather than to perform controllable dimension reduction.  

\paragraph{Bi-directional generative modeling.}  
Several bi-directional architectures have been proposed to jointly learn mappings between data and latent representations. Adversarially Learned Inference (ALI) and BiGAN~\citep{dumoulin2017adversariallylearnedinference} couple generator and inference networks through adversarial training. CycleGAN and MUNIT~\citep{huang2018multimodalunsupervisedimagetoimagetranslation} promote cycle consistency to learn unpaired inverse mappings. Multimodal VAEs (MVAE) extend VAEs with product-of-experts inference to handle missing modalities~\citep{wu2018multimodalgenerativemodelsscalable}. More recently, normalizing-flow approaches explicitly reduce reconstruction error through invertible transformations~\citep{lee2025latent}. However, these methods provide limited ability to regulate which information is preserved or discarded, hence lacking controllability.  

\paragraph{Optimal transport.}  
Optimal Transport (OT) offers a principled way to compare probability measures. Entropic regularization and the Sinkhorn algorithm~\citep{NIPS2013_af21d0c9} made OT scalable~\citep{peyré2020computationaloptimaltransport}, with further improvements such as stabilized sparse scaling~\citep{schmitzer2019stabilized}. Quadratic-form OT generalizes OT to quadratic objectives over couplings~\citep{wang2025quadraticformoptimaltransport}. Gromov–Wasserstein (GW) OT compares relational structures, with variants for averaging~\citep{peyre2016gromov} and fused forms that integrate features with geometry~\citep{vayer2018fusedgromovwassersteindistancestructured}. Recent work has studied the stability and algorithmic properties of entropic GW~\citep{rioux2024entropicgromovwassersteindistancesstability}. Nonetheless, existing OT methods remain confined to aligning discrete samples and generally lack mechanisms for generative modeling.  

\section{Preliminaries}
Before presenting the details of CPFM, we briefly review the necessary background in optimal transport and flow matching generative models. 

\subsection{Linear Entropic Optimal Transport}
Let $\mathcal{X} \subset \mathbb{R}^d$ and $\mathcal{Y} \subset \mathbb{R}^d$ denote the source and target domains. A transport cost $c:\mathcal{X}\times\mathcal{Y}\to\mathbb{R}$ measures the cost of moving mass between points. Given marginal distributions $\mu_{\mathcal{X}}$ and $\mu_{\mathcal{Y}}$, the Linear Entropic Optimal Transport (LEOT) problem seeks an optimal coupling $\pi$ that minimizes $
\inf_{\pi \in \Pi(\mu_{\mathcal{X}},\mu_{\mathcal{Y}})} 
   \int_{\mathcal{X}\times\mathcal{Y}} c(x,y)\, d\pi(x,y)
   + \varepsilon \,\mathrm{KL}\!\big(\pi \,\|\, \mu_{\mathcal{X}} \times \mu_{\mathcal{Y}}\big)$,
where $\Pi(\mu_{\mathcal{X}},\mu_{\mathcal{Y}})$ is the set of couplings with marginals $\mu_{\mathcal{X}}$ and $\mu_{\mathcal{Y}}$, $\varepsilon>0$ controls entropy regularization, and $\mathrm{KL}$ denotes Kullback–Leibler divergence.  

In the finite-sample case with $n$ support points, LEOT can be efficiently solved using the Sinkhorn algorithm, with per-iteration complexity $\mathcal{O}(n^2)$~\citep{NIPS2013_af21d0c9}.

\subsection{Quadratic Optimal Transport}
Quadratic Optimal Transport (QOT) generalizes the linear case by considering pairwise relations. It minimizes $
\inf_{\pi \in \Pi(\mu_{\mathcal{X}},\mu_{\mathcal{Y}})} 
   \iint_{\mathcal{X}^2 \times \mathcal{Y}^2} 
   c(x,x',y,y')\, d\pi(x,y)\, d\pi(x',y')$,
where $c:(\mathcal{X}\times\mathcal{X})\times(\mathcal{Y}\times\mathcal{Y}) \to \mathbb{R}^+$ takes inputs from two domains. Unlike LEOT, QOT naturally accommodates $\mathcal{X}$ and $\mathcal{Y}$ of different dimensions.  

In practice, finite-sample QOT with entropic regularization can be optimized via mirror descent, though with higher computational cost, typically $\mathcal{O}(n^3)$ per iteration~\citep{peyre2016gromov}. A widely used special case is Gromov–Wasserstein OT (GWOT), where the cost is defined as $
c_{\mathrm{GWOT}} = \big| d_\mathcal{X}(x,x') - d_\mathcal{Y}(y,y') \big|^2$, which
represents the distortion between intra-domain distances. Entropic GWOT admits an elegant variational formulation that enables alternating minimization with $\mathcal{O}(n^2)$ complexity~\citep{rioux2024entropicgromovwassersteindistancesstability}. In CPFM, we extend GWOT by introducing kernelized costs that encode richer forms of prior knowledge.

\subsection{Flow Matching}
\label{sec:flowmatching}
Flow matching (FM) provides a framework for learning generative models by transporting a base distribution $p_0$ (often $\mathcal{N}(0,I)$) to a target distribution $p_1$. The goal is to learn a time-dependent vector field $
u_\theta: \mathbb{R}^d \times [0,1] \to \mathbb{R}^d$ 
that induces an interpolating path $p_t$ between $p_0$ and $p_1$.  

Following the stochastic interpolants framework~\citep{albergo2023stochasticinterpolantsunifyingframework}, we construct reference paths by sampling $x(0) \sim p_0$, $x(1) \sim p_1$, and defining $x(t) = a_t x(0) + b_t x(1)$ for differentiable scalar functions $(a_t, b_t)$ with boundary conditions $(a_0,b_0)=(1,0)$ and $(a_1,b_1)=(0,1)$. The instantaneous conditional velocity is $
v_t(x(t); x(0),x(1)) := \tfrac{d}{dt} x(t) = \dot a_t x(0) + \dot b_t x(1)$,
and the optimal state-dependent drift is the conditional expectation $
u^*(x,t) = \mathbb{E}_{x(0),x(1)}\!\left[v_t(x(t); x(0),x(1)) \,\middle|\, x(t)=x \right]$.

A neural network $u_\theta$ is trained to approximate $u^*$ by minimizing the conditional flow-matching loss~\citep{lipman2023flowmatchinggenerativemodeling}:
\[
\mathcal{L}(\theta) 
= \mathbb{E}\!\left[\| u_\theta(x(t),t) - v_t(x(t)\mid x(0),x(1)) \|^2 \right],
\]
where the expectation is taken over $t \sim \mathrm{Uniform}[0,1]$, $x(0)\sim p_0$, and $x(1)\sim p_1$.  

During inference, sampling proceeds by drawing $x(0)\sim p_0$ and integrating the ODE $\tfrac{d}{dt}x(t)=u_\theta(x(t),t)$. generating final samples $x(1)$ that approximate $p_1$.  


\section{Method}
The goal of CPFM is to construct semantically meaningful samplers for the conditional distributions $p(x \mid y)$ and $p(y \mid x)$, given a training dataset $\mathcal{D}_x = \{x_i\}_{i=1}^n$. As discussed, our approach proceeds in two stages.  
\begin{figure}[h]
    \centering
\includegraphics[width=0.49\textwidth]{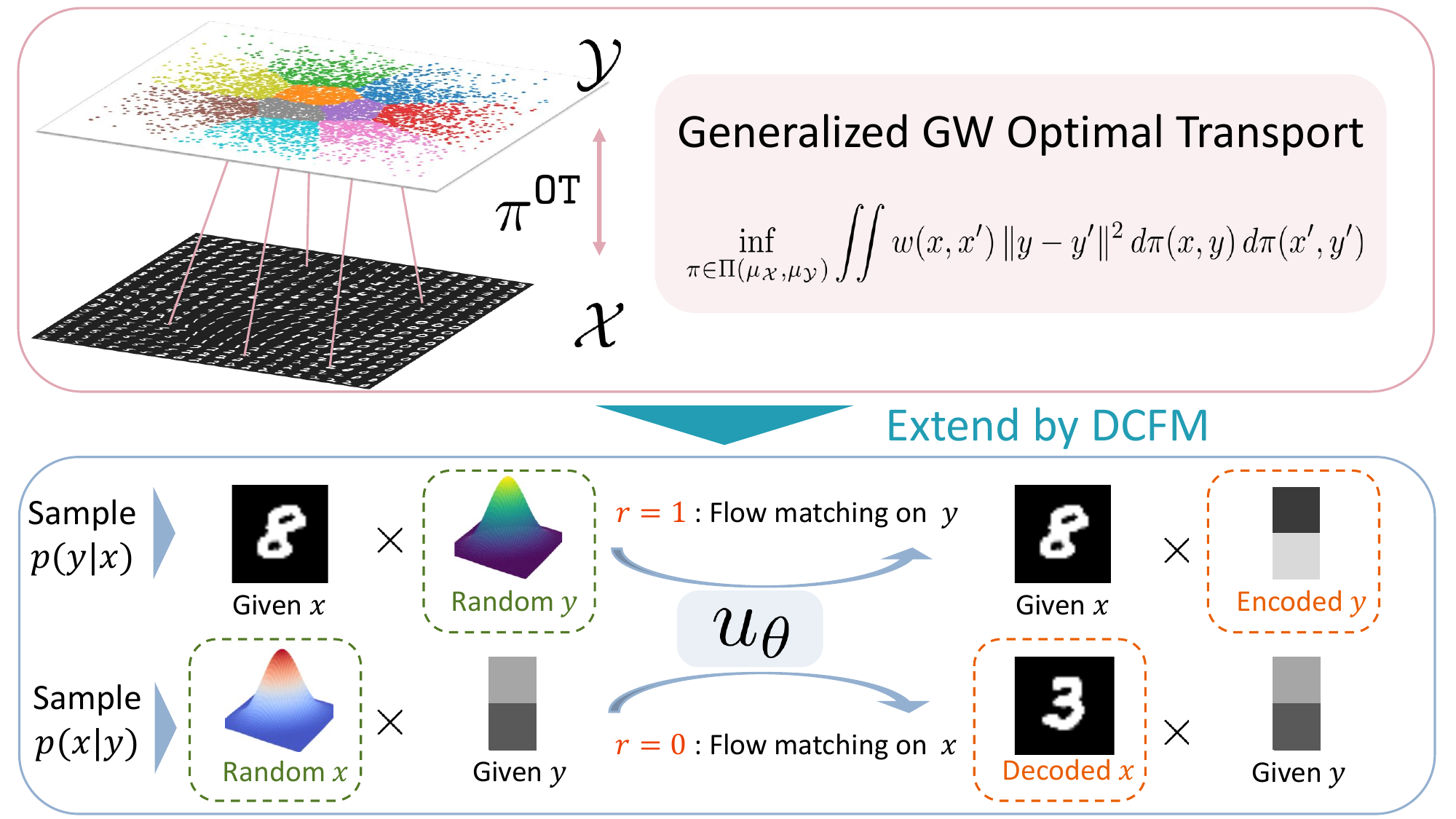}
    \caption{An overview of coupled flow matching.}
    \label{fig:sketch}
\end{figure}

In the \textit{first stage}, we draw $\mathcal{D}_y = \{y_i\}_{i=1}^n$ from the marginal distribution $\mu_\mathcal{Y}$ and establish a probabilistic correspondence between $\mathcal{D}_x$ and $\mathcal{D}_y$ via a generalized Gromov–Wasserstein optimal transport (GWOT). This coupling incorporates user-specified priors through kernelized costs to enable controllable alignment between data and embeddings. In the \textit{second stage}, we extend this coupling from the discrete training samples to the entire joint space $\mathcal{X} \times \mathcal{Y}$ by training a \emph{dual conditional flow matching} model. 

The overall pipeline of CPFM is illustrated in Figure~\ref{fig:sketch}. The following subsections present the technical details of each component. For brevity, proofs of propositions and theorems are deferred to the appendix. 

\subsection{Generalized Gromov-Wasserstein Optimal Transport}
Standard GWOT compares relational structures across domains but typically relies only on $\ell_p$ distances between raw data points. This makes it difficult to incorporate prior knowledge such as semantic labels, neighborhood graphs, or other task-relevant information. To overcome this limitation, we introduce a \emph{generalized} formulation that leverages kernel functions to encode flexible structures in the data space $\setX$.  

Formally, given marginals $\mu_\mathcal{X}$ and $\mu_\mathcal{Y}$, we introduce the following optimization problem
\begin{equation}
\label{eqn:otobjective}
\inf_{\pi \in \Pi(\mu_{\mathcal{X}}, \mu_{\mathcal{Y}})} 
    \iint_{\mathcal{X}^2 \times \mathcal{Y}^2} 
        k(x,x') \,\|y-y'\|^2 \, d\pi(x,y)\, d\pi(x',y'),
\end{equation}
where $k:\mathcal{X}\times\mathcal{X}\to\mathbb{R}$ is a symmetric kernel that captures user-specified relational information in $\mathcal{X}$. 

Formulation~\eqref{eqn:otobjective} encourages the semantic alignment between $(x,x')$ and $(y,y')$: when  $k(x,x')$ is large, meaning $x$ is similar to $x'$, $\|y-y'\|$ should be small, meaning $y$ is similar to $y'$ in the latent space. The semantic information is directly encoded into the transport objective through the kernel function, thus enabling practical flexibility.  

\begin{proposition}
Standard GWOT is a special case of~\eqref{eqn:otobjective} with $k(x,x') = -\|x-x'\|^2$.  
\end{proposition}

\paragraph{Finite-sample formulation.}  
In practice, we observe finite datasets $\mathcal{D}_x = \{x_i\}_{i=1}^n$ and $\mathcal{D}_y = \{y_j\}_{j=1}^n$, corresponding to empirical measures $\hat\mu_\mathcal{X} = \tfrac{1}{n}\sum_{i=1}^n \delta_{x_i}$ and $\hat\mu_\mathcal{Y} = \tfrac{1}{n}\sum_{j=1}^n \delta_{y_j}$. With a slight abuse of notation, the coupling $\pi \in \Pi(\hat\mu_\mathcal{X},\hat\mu_\mathcal{Y})$ can be represented as a matrix $\pi \in \mathbb{R}^{n\times n}$ with row-sum and column-sum equal to $\frac{1}{n}$, where $\pi_{ij}$ denotes the transported mass from $x_i$ to $y_j$. The finite-sample counterpart to~\eqref{eqn:otobjective} becomes
\begin{equation}
\label{eqn:discreteobjective}
\inf_{\pi \in \Pi(\hat\mu_\mathcal{X}, \hat\mu_\mathcal{Y})}
    \sum_{i,j,i',j'=1}^n 
    \pi_{ij}\pi_{i'j'} \, k(x_i,x_{i'}) \,\|y_j-y_{j'}\|^2.
\end{equation}
We denote its optimal solution by $\piot$.  

\paragraph{Variational reformulation.}  
Although~\eqref{eqn:discreteobjective} seems difficult to optimize as it involves quadratic terms in entries of $\pi$, we show that it admits a tractable variational representation. Let $G \in \mathbb{R}^{n\times n}$ be the kernel Gram matrix with entries $G_{ij}=k(x_i,x_j)$. Assuming it has rank $m$, we can write its Cholesky factorization as $G = \Phi^\top \Phi$ with $\Phi \in \mathbb{R}^{m\times n}$. The following theorem proposes a reformulation of~\eqref{eqn:discreteobjective}.

\begin{theorem}
\label{thm:usefulone}
The optimal transport plan $\piot$ is also the minimizer of
\begin{equation}
\label{eqn:discrete}
\begin{aligned}
&\inf_{\pi\in\Pi(\hatmu_{\setX},\hatmu_\setY)} \inf_{A\in \mathbb{R}^{m\times d_y}} 
\Biggl\{ \|A\|^2 \\
&+ \sum_{i=1}^n\sum_{j=1}^n 
    \pi_{ij}\Bigl(\|y_j\|^2 w(x_i) 
    - 2 \langle A, \Phi_i y_j^\top \rangle \Bigr)\Biggr\},
    \end{aligned}
\end{equation}
where $A$ is an auxiliary matrix, $\Phi_i$ is the $i$-th column of $\Phi$, $w(x_i) = \sum_{j=1}^n k(x_i,x_j)$, and the matrix inner product $\inner{\cdot}{\star}$ is defined as the inner product of their flattened vectors.
\end{theorem}

Theorem~\ref{thm:usefulone} transforms the quadratic coupling into a problem linear in $\pi$, at the cost of introducing an auxiliary variable $A$. This reformulation simplifies the optimization algorithm design.  

\paragraph{Alternating minimization.}  
Building on the variational form~\eqref{eqn:discrete} of generalized GWOT, we develop an alternating optimization algorithm to obtain $\piot$. 

\ul{Step 1: Fix $\pi$, update $A$.} For fixed $\pi$, the subproblem is a quadratic function of $A$, with the optimal solution
\[
A^\star = \sum_{i=1}^n \sum_{j=1}^n \pi_{ij}\,\Phi_i y_j^\top .
\]

\ul{Step 2: Fix $A$, update $\pi$.} For fixed $A$, the cost reduces to a linear assignment problem 
\begin{equation}
\label{eq:subobj-pi}
\min_{\pi \in \mathbb{R}^{n\times n}} 
\sum_{i=1}^n\sum_{j=1}^n 
\pi_{ij}\,
\underbrace{\Bigl(\|y_j\|^2 w(x_i) 
- 2 \langle A, \Phi_i y_j^\top \rangle\Bigr)}_{=:~c(x_i,y_j)}.
\end{equation}

We add an entropic regularization $\varepsilon \sum_{i=1}^n \sum_{j=1}^n \pi_{ij}\big(\log \pi_{ij} - 1\big)$ to ensure numerical tractability, where $\varepsilon>0$ controls the strength of the regularization.
The resulting regularized objective becomes an LEOT problem that can be solved by invoking Sinkhorn iterations:
\[
\pi^\star \;\approx\; \mathrm{Sinkhorn}(c(x_i,y_j),\,\varepsilon).
\]

\begin{algorithm}[t]
\caption{Alternating optimization for~\eqref{eqn:discrete}}
\label{alg:opt}
\KwIn{Kernel factors $\Phi \in \mathbb{R}^{n\times m}$, auxiliary matrix $A$,
embedding matrix $Y \in \mathbb{R}^{n\times d_y}$, tolerance $\tau > 0$, regularization $\varepsilon > 0$}
\KwOut{transport plan $\pi \in \mathbb{R}^{n\times n}$, auxiliary matrix $A$, OT objective $\mathcal{L}_{\mathrm{OT}}$}

Initialize $\mathcal{L}_{\mathrm{OT}}^{(0)} \leftarrow +\infty$\;
Compute $Y_{\mathrm{norm}} \leftarrow \big(\|y_1\|^2,\dots,\|y_n\|^2\big)^\top$

Compute $w\leftarrow\text{Rowsum}(\Phi\Phi^\top)$

\While{True}{
    $C \leftarrow w\,Y_{\mathrm{norm}}^\top - 2\, \Phi A Y^\top$ \tcp{cost matrix}

    $\pi \leftarrow \mathrm{Sinkhorn}(C, \varepsilon)$ 

    $A \leftarrow \Phi^\top \pi Y$ 

    $\mathcal{L}_{\mathrm{OT}}^{(t)} \leftarrow \|A\|_F^2 + \langle C, \pi \rangle$ \tcp{current OT value}
    
    \If{$\mathcal{L}_{\mathrm{OT}}^{(t-1)} - \mathcal{L}_{\mathrm{OT}}^{(t)} < \tau$}{
        \textbf{break}
    }
}
\Return{$(\pi, A, \mathcal{L}_{\mathrm{OT}}^{(t)})$}
\end{algorithm}

Based on the alternating optimization scheme described above, we have the following proposition that certifies the convergence and per-iteration complexity of Algorithm~\ref{alg:opt}.

\begin{proposition}
\label{thm:con}
Algorithm~\ref{alg:opt} produces a non-increasing sequence of OT objectives $\{\mathcal{L}_{\mathrm{OT}}^{(t)}\}$ that converges to a stationary point, with per-iteration complexity $\mathcal{O}(n^2)$.
\end{proposition}

\paragraph{Practical considerations.}  
In practice, a small entropic regularization parameter $\varepsilon$ yields a more accurate approximation to the original~\eqref{eqn:discrete}, but slows down the convergence of the Sinkhorn subroutine. To address this, we develop an $\varepsilon$-scheduling scheme~\citep{schmitzer2019stabilized}. Specifically, we start from a larger $\varepsilon$, and then iteratively reduce it via bisection until reaching floating-point precision. This approach balances convergence speed and numerical stability. A detailed description of our $\varepsilon$-scheduling scheme is discussed in the supplementary material.

\subsection{Dual Conditional Flow Matching}
While $\pi^{\mathrm{OT}}$ provides a controllable correspondence between $\mathcal{D}_x$ and $\mathcal{D}_y$, it is defined only on finite samples. To generalize beyond the training support and construct conditional samplers on the entire space $\mathcal{X}\times \mathcal{Y}$, we introduce \emph{Dual Conditional Flow Matching} (DCFM), the second pillar of CPFM.  

\paragraph{Dual conditional flows.}  
DCFM learns probability flows in both spaces.  

To sample from $p(y| x)$ given $x$, we first draw $y(0) \sim p^y_0$ from a base distribution and integrate the ODE
\[
\frac{d}{dt} y(t) = u_y(y(t);x,t)
\]
to obtain $y(1)$, which serves as a sample from $p(y| x)$.  

Conversely, to generate from $p(x | y)$ given an embedding $y$, we draw $x(0) \sim p^x_0$ and integrate
\[
\frac{d}{dt} x(t) = u_x(x(t);y,t)
\]
to obtain $x(1)$, which serves as a sample from $p(x\mid y)$.  

The key observation is that $u_x$ and $u_y$ are structurally symmetric. We therefore parameterize both with a single neural network $u_\theta$, which outputs the drift for both $x$ and $y$. Specifically, we introduce a role flag $r \in \{0,1\}$:  
\[
u_{\theta}(\cdot, r) = 
\begin{cases}
u_x(\cdot)\,, & r=0, \\
u_y(\cdot)\,, & r=1.
\end{cases}
\]
In practice, $u_\theta$ is implemented as a shared backbone with two decoding heads, selected by $r$. This design leverages symmetry, avoids parameter redundancy, and enables joint parameter sharing.

\paragraph{Training objective.}  
Training balances two objectives: learning $u_x$ and learning $u_y$. We use $a_t$ and $b_t$ to denote two differentiable functions over $t$ that satisfy the boundary condition specified in Section~\ref{sec:flowmatching}.We first sample $(x(1),y(1)) \sim \pi^{\mathrm{OT}}$ from the optimal coupling.  

If $r=0$, we generate an interpolant path for $x$: sample $x(0)\sim p^x_0$, set $x(t)=a_t x(0)+b_t x(1)$, and define velocity $v_x(t)=\dot a_t x(0)+\dot b_t x(1)$. The conditional flow-matching loss is
\begin{equation}
\label{eqn:ellxdef}
\ell_x(u,t,x(t),y(1),v_x(t))
= \|u(x(t),y(1),t,0) - v_x(t)\|^2.
\end{equation}

If $r=1$, we instead interpolate $y$: sample $y(0)\sim p^y_0$, set $y(t)=a_t y(0)+b_t y(1)$, and define $v_y(t)=\dot a_t y(0)+\dot b_t y(1)$. The loss becomes
\begin{equation}
\label{eqn:ellydef}
\ell_y(u,t,x(1),y(t),v_y(t))
= \|u(x(1),y(t),t,1) - v_y(t)\|^2.
\end{equation}

The full DCFM training objective is a weighted combination of the two roles:
\begin{equation}
\label{eqn:dcfmloss}
\mathcal{L}_{\mathrm{DCFM}}(u)
= (1-\alpha)\,\mathbb{E}[\ell_x(u)] 
+ \alpha\,\mathbb{E}[\ell_y(u)],
\end{equation}
where $\alpha \in (0,1)$ balances the two objectives. Expectations are taken over random draws of $(t,x,y)$ and the role indicator $r \sim \mathrm{Bernoulli}(\alpha)$.  

\paragraph{Algorithmic summary.}  
The stochastic training and sampling procedure is summarized in Algorithm~\ref{alg:train}.

\begin{algorithm}[h]
\caption{Training DCFM by optimizing~\eqref{eqn:dcfmloss} }
\label{alg:train}
\KwIn{Dataset $\setD_x$, embedding dataset $\setD_y$ sampled from $\mu_{\setY}$, transport plan $\piot \in \mathbb R^{|\setD_x|\times|\setD_x|}$, Bernoulli parameter $\alpha$}
\KwOut{trained drift model $u_\theta$}

\While{not converged}{
    Sample a pair $(x(1),y(1))$ based on $\piot$
    
    Sample role $r \sim \mathrm{Bernoulli}(\alpha)$ \tcp*{decide train $x$ or train $y$.}

    Sample time $t \sim \mathrm{Uniform}[0,1]$ 
    
    \eIf{$r = 0$}{
        Sample $x(0) \sim p_0^x$ 
        
        Compute interpolant $x(t) = a_t x(0) + b_t x(1)$ 
        
        Compute velocity $v_x(t) = \dot a_t x(0) + \dot b_t x(1)$ 
        
        Compute loss $\ell(u_\theta) = \ell_x(u_\theta)$ defined in~\eqref{eqn:ellxdef}
    }{
        Sample $y(0) \sim p_0^y$ 
        
        Compute interpolant $y(t) = a_t y(0) + b_t y(1)$ 
        
        Compute velocity $v_t^y = \dot a_t y(0) + \dot b_t y(1)$ 
        
        Compute loss $\ell(u_\theta) = \ell_y(u_\theta)$ defined in~\eqref{eqn:ellydef}
    }
    
 Calculate $\nabla_\theta\ell(u_\theta)$ and update $\theta$ by AdamW.
}
\end{algorithm}

We now argue that the two objectives in~\eqref{eqn:dcfmloss} are not in conflict. In fact, the joint loss admits a solution that simultaneously approximates both conditional flows.  
\begin{theorem}\label{thm:dec}
Let $u^*$ denote the drift field corresponding to the optimal solution of~\eqref{eqn:dcfmloss}. Then $u^*$ satisfies
$u^*(z,c,t,0)=\mathbb E[\dot a_t x(0)+\dot b_t x(1) \mid x(t)=z,\,y(1)=c]$, and $u^*(c,z,t,1)=\mathbb E[\dot a_t y(0)+\dot b_t y(1) \mid y(t)=z,\, x(1)=c]$.
\end{theorem}

This theorem shows that minimizing the combined loss~\eqref{eqn:dcfmloss} yields a drift field that is consistent with both conditional velocity fields. 
Thus, integrating the $\dfrac{d}{dt} x(t)=u^*(x(t),y(1),t,0)$ or $\dfrac{d}{dt} y(t)=u^*(x(1),y(t),t,1)$ produces samples from $p(y| x)$ and $p(x| y)$ without interference between the two directions.

The inference procedures are summarized in Algorithm~\ref{alg:inference}.

\begin{algorithm}[h]
\caption{Inference with DCFM}
\label{alg:inference}
\KwIn{trained drift model $u_\theta$, conditioning variable $c$ (either $x$ or $y$), direction $r \in \{0,1\}$, number of integration steps $T$}
\KwOut{sample $\hat z$ in the target space}

\eIf{$r = 0$}{
    Sample initial $x_0 \sim p_0^x$\;
    Set $z_0 \gets x_0$\;
}{
    Sample initial $y_0 \sim p_0^y$\;
    Set $z_0 \gets y_0$\;
}

\For{$k=0$ \KwTo $T-1$}{
    Set $t = k/T$\;
    Compute drift $u = u_\theta(z_k,c,t,r)$ if $r=0$ or $u=u_\theta(c,z_k,t,r)$ if $r=1$\;
    Update state $z_{k+1} \gets z_k + \frac{1}{T}\cdot u$ \tcp*{Euler step}
}

Return $\hat z = z_T$\;
\end{algorithm}

\section{Experiments}
\label{sec:experiments}

In this section, we evaluate CPFM on four image generation datasets and one molecule generation dataset to showcase its performance in both dimensionality reduction and sample reconstruction. The implementation and training code for DCFM are available in an anonymous repository: \url{https://anonymous.4open.science/r/AISTATS-CEF9/}. We use two NVIDIA A10 GPUs for training $u_\theta$.

For all experiments, we set the embedding space $\mathcal{Y}$ to be two-dimensional ($\mathbb{R}^2$). This choice enables direct visualization of the learned embeddings and provides a stringent test case: mapping high-dimensional data into such a severely compressed space poses a challenging scenario for both representation learning and faithful reconstruction.   




\subsection{MNIST}
We begin with the MNIST dataset~\citep{deng2012mnist}, a benchmark collection of handwritten digit images ($0$–$9$), to illustrate the effectiveness of CPFM.  

\paragraph{Kernel construction.}  
Since MNIST is labeled, we design a kernel function that incorporates both visual similarity and label information. Inspired by Conditional Random Field models~\citep{efficientcrf}, we propose the composite kernel
\begin{equation}
\label{eqn:mnistkernel}
k_{\text{image}}(x_i, x_j) 
= \underbrace{\exp\!\left(-\tfrac{\|x_i - x_j\|^2}{2\sigma^2}\right)}_{\text{appearance kernel}}
\cdot 
\underbrace{\mathbf{1}\{\, l_i = l_j \,\}}_{\text{label kernel}},
\end{equation}
where $l_i$ and $l_j$ are the digit labels of images $x_i$ and $x_j$, $\mathbf{1}\{\cdot\}$ is an indicator function, and $\sigma$ is the bandwidth parameter estimated as $
\sigma = \frac{1}{n^2}\sum_{i,j=1}^n \|x_i - x_j\|$.

Here, the appearance kernel is a Gaussian (heat) kernel that captures pixel-level similarity, while the label kernel encodes semantic consistency by enforcing affinity only among images of the same digit. This construction ensures that the transport cost reflects both geometric similarity and label-based priors.  

\paragraph{Generalized GWOT.}  
With the kernel in~\eqref{eqn:mnistkernel}, we compute the Gram matrix of the MNIST images, draw latent samples $y$ from $\mu_\mathcal{Y}$, and apply Algorithm~\ref{alg:opt} to solve the generalized GWOT problem. 
The resulting embeddings are plotted in Figure~\ref{fig:MNIST_OT}. 

\begin{figure}[h]
    \centering
    \begin{subfigure}[t]{0.15\textwidth}
        \centering
        \includegraphics[width=\textwidth]{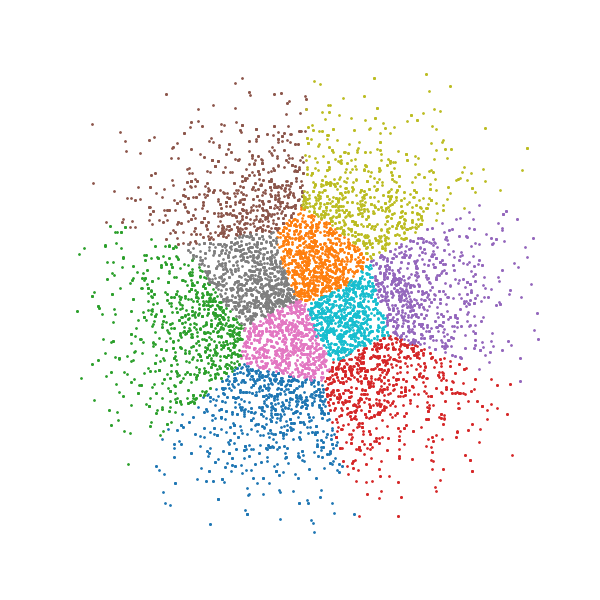}
    \end{subfigure}
    \begin{subfigure}[t]{0.15\textwidth}
        \centering
        \includegraphics[width=\textwidth]{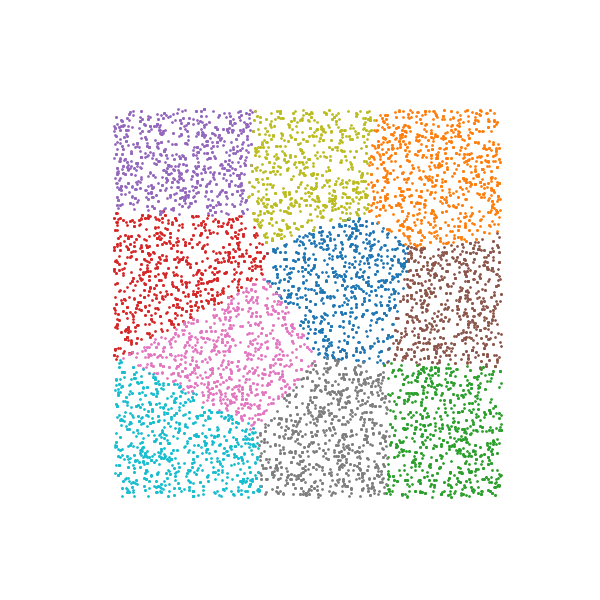}
    \end{subfigure}
    \begin{subfigure}[t]{0.15\textwidth}
        \centering
        \includegraphics[width=\textwidth]{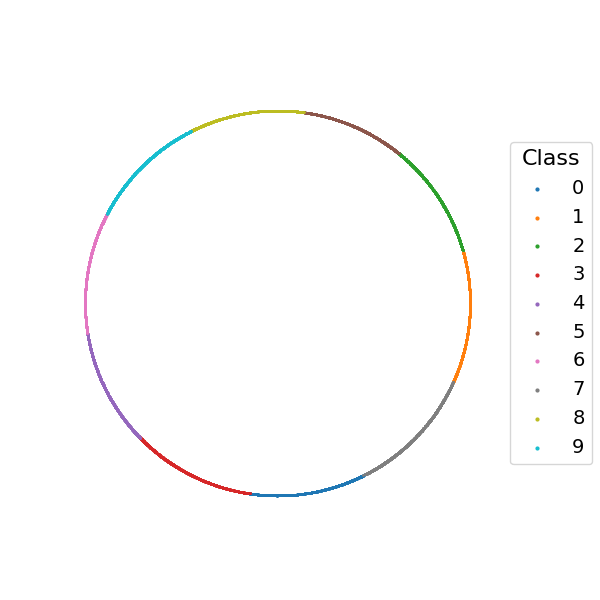}
    \end{subfigure}

    \caption{Based on the generalized GWOT transport plan $\pi$, each source sample $x$ is associated with a probability distribution over candidate embeddings ${y}$. A single embedding $y$ is randomly sampled according to its assigned weight $\pi_i$. The latent embedding distributions $\mu_{\setY}$ considered are: (Left) a Gaussian $\mathcal{N}(0, I_2)$; (Middle) a uniform distribution on the square $[-1,1]^2$; and (Right) a uniform distribution on the unit circle ${, y \in \mathbb{R}^2 : |y|=1 ,}$.
    }
    \label{fig:MNIST_OT}
\end{figure}

The generated embeddings align closely with the contour of the prescribed target distributions $\mu_\setY$, demonstrating that generalized GWOT can faithfully adapt to different latent marginals. Moreover, the embeddings exhibit well-separated clusters for different digit classes, reflecting the influence of the label-sensitive kernel. This separation indicates that Algorithm~\ref{alg:opt} could preserve geometric structure and semantic priors simultaneously.

\paragraph{DCFM reconstruction.}  
Building on the transport plan provided by generalized GWOT, we train DCFM using Algorithm~\ref{alg:train} to jointly perform dimensionality reduction into two dimensions and reconstruction back to the image space. After training, we sample $100$ random images, embed them into $\mathbb{R}^2$, and then reconstruct from the embeddings using the learned drift network $u_\theta$. Figure~\ref{fig:MNIST_embd_construct} shows both the original and reconstructed images. For clarity of visualization, the figure displays $1024$ embeddings in the latent space rather than $100$.

\begin{figure}[h]
    \centering
    \includegraphics[width=0.49\textwidth]{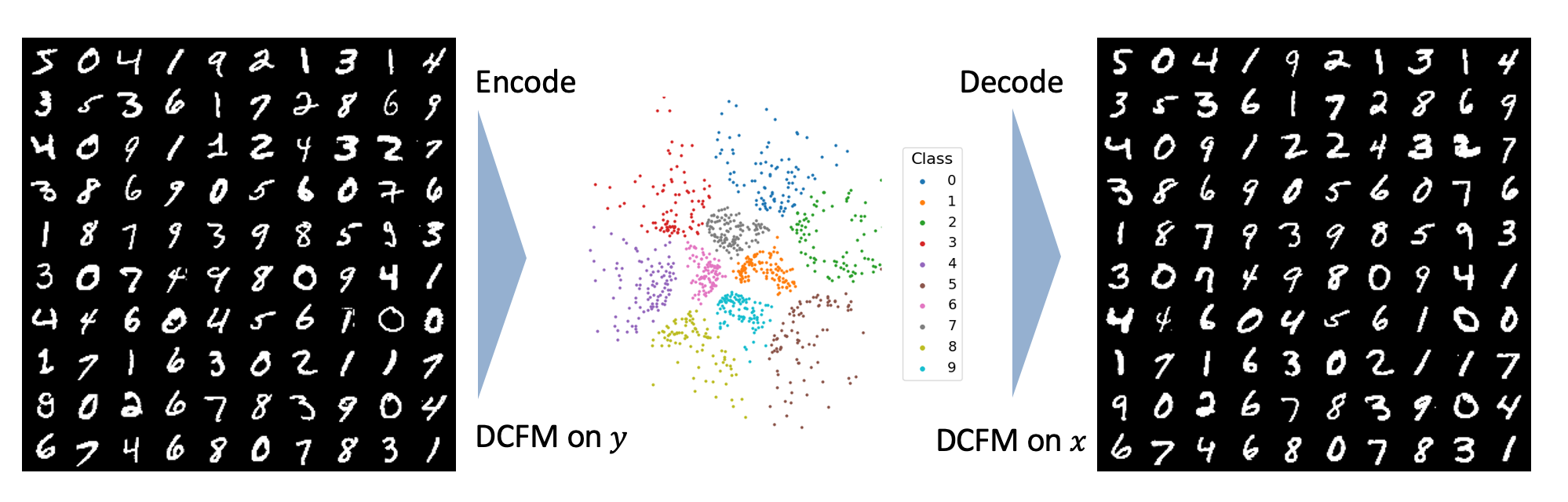}
    \caption{The MNIST dataset compressed into two dimensions by DCFM, 
    where different classes are distinguished by labels and then reconstructed. 
}
    \label{fig:MNIST_embd_construct}
\end{figure}

From Figure~\ref{fig:MNIST_embd_construct}, we observe that reconstructed digits preserve their class identity, underscoring DCFM’s ability to faithfully retain label information encoded in the discrete transport plan $\pi^{\mathrm{OT}}$. At the same time, the reconstructions exhibit subtle variations in handwriting style compared to the originals. This suggests that DCFM not only captures class-relevant features but also introduces diversity in nuisance attributes.

\subsection{Image datasets}
Beyond MNIST, we further evaluate CPFM on three widely used image benchmarks: \textbf{CIFAR-10}~\citep{krizhevsky2009learning}, \textbf{TinyImageNet}~\citep{le2015tiny}, and \textbf{AFHQ}~\citep{choi2020stargan}. Since we pretrain flow matching model $u_\theta$ from random initialization, we downsample the resolution of TinyImagenet and AFHQ to $64\times 64$, and $128\times 128$, respectively, to reduce the carbon footprint in the training process. 

\paragraph{Evaluation metrics}
\begin{table*}[h]
\centering
\begin{tabular}{ccccc}
   \toprule
   Distance to Gaussian $\downarrow$ & MNIST & CIFAR-10 & AFHQ & TinyImagenet \\
   \midrule
   KPCA & 245 $\pm$ 4 & 720 $\pm$ 60 & 48000 $\pm$ 5900 & 47000 $\pm$ 5000   \\
   VAE & 32 $\pm$ 5 & 52 $\pm$ 7 & 48 $\pm$ 7 & 49 $\pm$ 8 \\
   DiffAE & 4.96 $\pm$  0.04  & 7.40 $\pm$ 0.02 &28 $\pm$ 2 & 20.3 $\pm$ 0.2 \\
   Info-Diffusion & \ul{0.1608} $\pm$ 1e-06 & \ul{0.3092} $\pm$ 2e-06 & \ul{0.2996} $\pm$ 2e-07 & \ul{0.253}  $\pm$ 0.006  \\
   \midrule
   \textbf{CPFM} & \textbf{0.113} $\pm$ 0.002 &\textbf{ 0.17} $\pm$ 0.01 & \textbf{0.138} $\pm$ 0.008 & \textbf{0.25} $\pm$ 0.03 \\
   \bottomrule
   
\end{tabular}
\caption{Distance to Gaussian. Mean $\pm$ std calculated over 5 independent generations.}
\label{Table:OTandFid}
\vspace{-0.2cm}
\end{table*}

The primary evaluation metrics we use is the \textit{Wasserstein distance of the latent distribution to Gaussian distribution}.

\paragraph{Benchmarks.}  
We compare CPFM against a set of representative baselines, including KPCA~\citep{gedon2023invertible}, VAE~\citep{Kingma_2019}, DiffAE~\citep{diffusionae}, and Info-Diffusion~\citep{wang2023infodiffusion}. To ensure fairness, all diffusion and flow matching-based models are trained with the same network architecture and number of training iterations. Quantitative results are summarized in Table~\ref{Table:OTandFid}.   

From Table~\ref{Table:OTandFid}, it is evident that CPFM achieves the lowest FID and OT loss, which again confirms its superior dimension reduction and sample reconstruction capability. It is worth noting that the reported FID scores are substantially higher than those in the original papers. This discrepancy is expected, as we deliberately impose an extreme dimension reduction setting by restricting the latent space to only two dimensions. 

\subsection{QM9}
To further evaluate CPFM beyond image domains, we test it on the QM9 dataset~\citep{qm91,qm92}, a widely used quantum chemistry benchmark. We randomly sample 4,000 molecules for training and 1,000 for evaluation. The molecule attribute of interest is the \textit{dipole moment} (a continuous variable, measured in Debye). Each molecule is encoded with \textsc{SELFIES}~\citep{Krenn_2020}.
 A VAE is first trained to obtain 64-dimensional latent representations of molecules, which serve as $\setX$.  

\paragraph{Kernel construction.}  
We design a composite kernel that integrates both structural similarity and property similarity:
\begin{equation}
\label{eqn:molkernel}
k_{\text{mol}}(i,j) 
= \tfrac{1}{2}\,
\underbrace{\frac{\langle f_i, f_j\rangle}
       {\|f_i\|_1 + \|f_j\|_1 - \langle f_i, f_j\rangle}}_{\text{structure kernel}}
+ \tfrac{1}{2}\,
\underbrace{|\ell_i - \ell_j|}_{\text{property kernel}},
\end{equation}
where $f_i$ is the binary Morgan fingerprint~\citep{doi:10.1021/ci100050t} of molecule $i$, used to compute the Tanimoto similarity~\citep{Bajusz2015Tanimoto}, and $\ell_i$ is its dipole moment \citep{article_dipole}.
\vspace{-0.2cm}
\paragraph{Dimension reduction.}  
We apply generalized GWOT with the kernel in~\eqref{eqn:molkernel}, producing two-dimensional embeddings. The results are visualized in Figure~\ref{fig:real}. Compared with standard t-SNE, the embeddings obtained via Algorithm~\ref{alg:opt} exhibit smoother transitions in dipole moment values, indicating that the latent coordinates align more naturally with the underlying property. This highlights CPFM’s ability to produce chemically meaningful embeddings.

\begin{figure}[h]
\vspace{-0.2cm}
    \centering
    \begin{subfigure}[t]{0.2\textwidth}
        \centering
        \includegraphics[width=\textwidth]{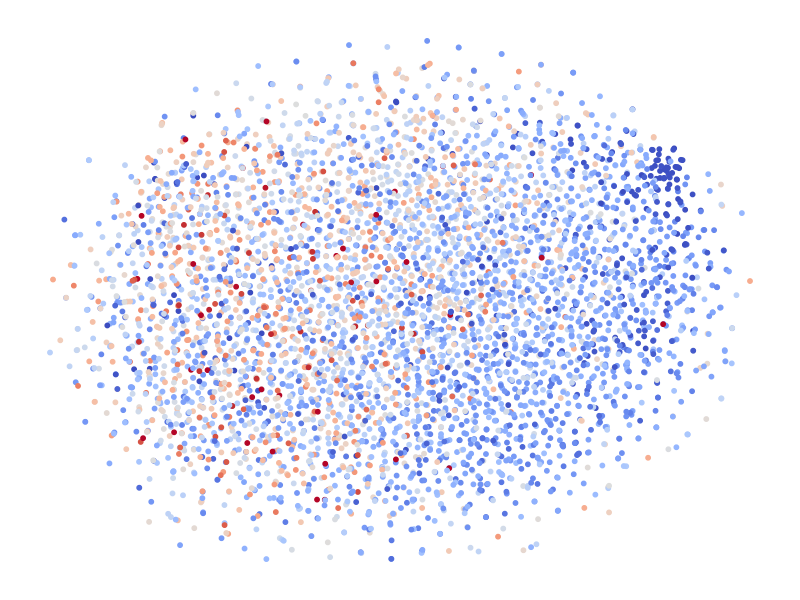}
    \end{subfigure}
    \begin{subfigure}[t]{0.2\textwidth}
        \centering
        \includegraphics[width=\textwidth]{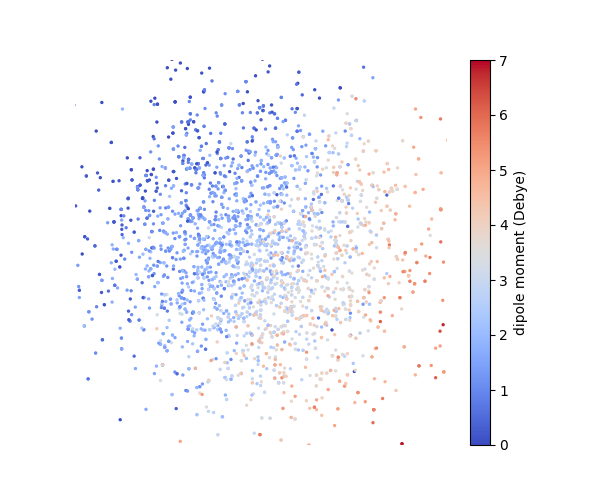}
    \end{subfigure}

    \caption{(Left) t-SNE visualization of the VAE latent space, where the labels do not show a clear separation. 
(Right) Embedding obtained from generalized GWOT, where the molecular embeddings vary continuously in label.
}
\vspace{-0.6cm}
    \label{fig:real}
\end{figure}

\vspace{-0.2cm}
\paragraph{Controlled generation.}  
Next, we train DCFM to learn a bidirectional mapping between the VAE latent space and the generalized GWOT embeddings. Figure~\ref{fig:real_recon} shows six examples of molecules and their reconstructions.  
\begin{figure}[h]
\vspace{-0.2cm}
    \centering
    \includegraphics[width=0.49\textwidth]{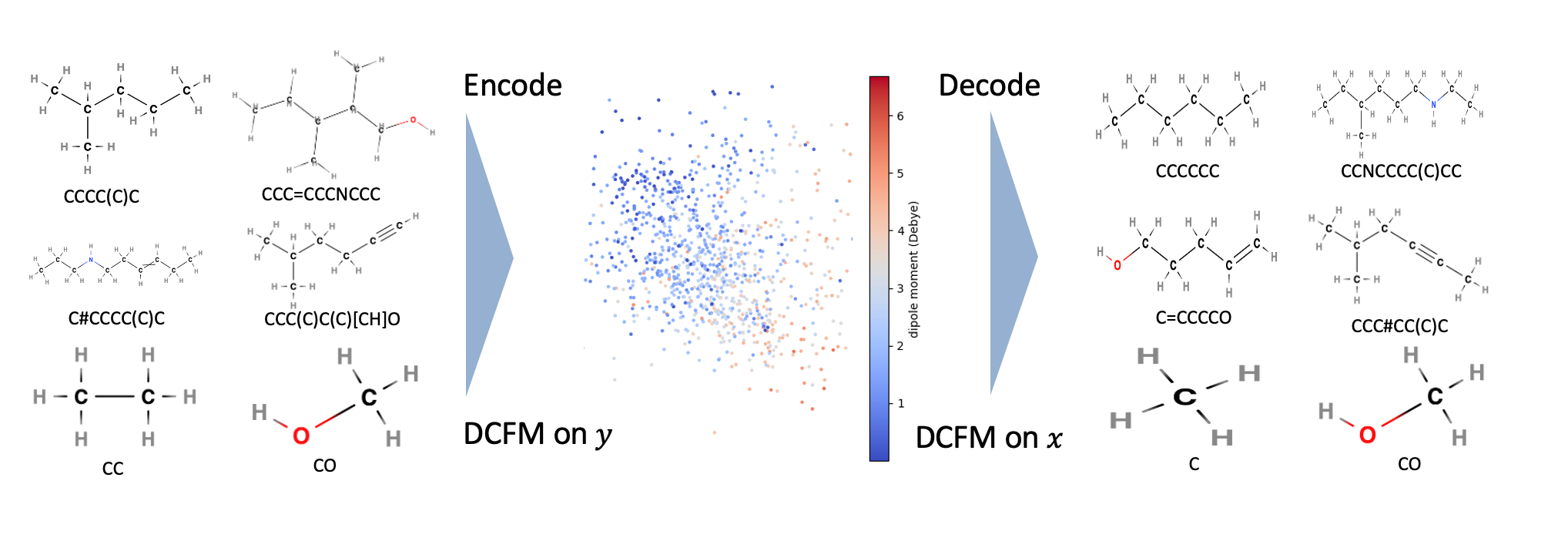}
    \caption{The QM9 dataset compressed into two dimensions by DCFM and then reconstructed. 
}
    \label{fig:real_recon}
    \vspace{-0.6cm}
\end{figure}

The reconstructions largely preserve atomic composition and molecular topology, demonstrating CPFM’s ability to retain chemically meaningful features even under drastic dimensionality reduction. 

\section{Discussion and Limitations}
This work introduced CPFM, a framework for controllable dimensionality reduction and reconstruction. While effective, training dual conditional flows from scratch remains computationally intensive on high-resolution datasets, suggesting the need for more efficient strategies such as multi-scale architectures.



\bibliographystyle{apalike}
\bibliography{reference}


\clearpage
\appendix
\thispagestyle{empty}

\onecolumn
\aistatstitle{Appendix}
The supplementary materials are organized as follows. Section~\ref{sec:practicalconsideration} elaborates on practical considerations of implementing CPFM. Section~\ref{sec:missingproof} presents the proofs to the theorems in the main paper. Section~\ref{sec:evaluationmetric} introduces the numerical evaluation metrics in detail. Finally, Section~\ref{sec:additionalexperiment} provides additional experiment results, including additional visualizations that illustrate the semantic coherence of the low-dimensional embeddings extracted by CPFM, visualizations of the DCFM architecture, and the choice of hyperparameters.
\section{Practical considerations}
\label{sec:practicalconsideration}
\subsection{$\varepsilon$-scheduling}
We improve Algorithm~\ref{alg:opt} by adding an adaptive schedule for $\varepsilon$ to balance entropy magnitude and numerical stability. A larger $\varepsilon$ makes the Sinkhorn subroutine updates faster and more stable, but it increases bias towards blurry assignments in $\piot$. A smaller $\varepsilon$ matches the unregularized quadratic objective better, but it may incur instability issues under floating-point arithmetic. Algorithm~\ref{alg:opt-adaptive} keeps the inner iteration stable while improving outer iterations through warm starts. Specifically, it maintains an interval between the last stable value $\varepsilon$ and a trial value $\varepsilon_{\mathrm{current}}$. If the inner call reaches the tolerance $\tau$ with no precision issues, we accept the iterate $(\pi,A)$ as a warm start and then halve $\varepsilon_{\mathrm{current}}$. If a precision limit is detected, we backtrack by moving $\varepsilon_{\mathrm{current}}$ to the midpoint between the failed value and the last stable value. The procedure stops when $|\varepsilon_{\mathrm{current}}-\varepsilon|<\delta$ and we return the transport plan computed at the last stable $\varepsilon$.
\begin{algorithm}[h]
\caption{Alternating optimization with adaptive entropic regularization}
\label{alg:opt-adaptive}
\KwIn{kernel weights $w \in \mathbb{R}^n$, feature matrix $\Phi \in \mathbb{R}^{n\times m}$, 
embedding matrix $Y \in \mathbb{R}^{n\times q}$, tolerance $\tau > 0$, initial regularization $\varepsilon > 0$, regularization tolerance $\delta$}
\KwOut{transport plan $\pi \in \mathbb{R}^{n\times n}$, OT objective $\mathcal{L}_{\mathrm{OT}}$}

Initialize $\varepsilon_{\text{current}} \leftarrow \varepsilon$\;
Initialize $A \leftarrow 0$\;

\While{True}{
    \If{Algorithm 1 $(w, \Phi, A, Y, \tau, \varepsilon_{\text{current}})$ reaches precision limitation}{
        $\varepsilon_\text{current}, \varepsilon \leftarrow \frac{\varepsilon + \varepsilon_\text{current}}{2},\varepsilon_\text{current}$\tcp*{bi-section search to determine the smallest stable $\varepsilon$.}
    }
    \Else{
        $\pi, A, \mathcal{L}_{\mathrm{OT}} \leftarrow \text{Algorithm 1}(w, \Phi, A, Y, \tau, \varepsilon_{\text{current}})$\;
        $\varepsilon_\text{current}, \varepsilon\leftarrow \frac{ \varepsilon_\text{current}}{2}, \varepsilon_\text{current}$\tcp*{halve $\varepsilon$.}
    }
    \If {$\big|\varepsilon_{\text{current}} - \varepsilon\big| < \delta$}{
        break\;
    }
}
\Return{$(\pi, \mathcal{L}_{\mathrm{OT}}^{(t)})$}
\end{algorithm}

In all experiments, we implement the generalized GWOT optimization with Algorithm~\ref{alg:opt-adaptive} and set the initial $\varepsilon$ to 0.01.

\subsection{Choice of $m$}
In the variational reformulation, we compute the feature map $\Phi$ using the Pivoted Cholesky feature map. Given the kernel Gram matrix $G=[k(x_i,x_j)]_{i,j=1}^n$, the algorithm compute the feature vector $\Phi_i\in\mathbb{R}^{m}$ for each sample such that
\[
\Phi^\top \Phi \approx G,
\]
where $\Phi = [\Phi_1, \ldots, \Phi_n]$.

Since $n$ can be huge in practice, Pivoted Cholesky factorization automatically selects rank $m$ to reach a target level of explained variance ratio $\eta\in[0,1]$. 
In all experiments, we choose the target explained variance ratio $\eta=0.95$.

\subsection{Interpolation of $\piot$} Despite the variational formulation introduced, Algorithm~\ref{alg:opt} still incurs $O(n^2)$ per-iteration complexity. This becomes infeasible in our computational infrastructure when  $n\ge10{,}000$. We therefore subsample $10{,}000$ points from the entire dataset, compute $\piot$ on this subset, and interpolate the plan to the full set to obtain $\widehat{\piot}$. For a point $x$, let $\mathcal{N}_k(x)$ be its $k$ same-class neighbors in the subset and define weights
\[
\alpha_j(x)=\frac{\exp\!\big(-\|x-x_j\|_2\big)}{\sum_{m\in\mathcal{N}_k(x)}\exp\!\big(-\|x-x_m\|_2\big)},\quad j\in\mathcal{N}_k(x).
\]
The row-interpolated transport plan is
\[
\widehat{\piot}(x,\cdot)=\sum_{j\in\mathcal{N}_k(x)} \alpha_j(x)\,\piot(x_j,\cdot).
\]
During DCFM training, we sample from a mixture of $\piot$ and the interpolated plan $\widehat{\piot}$.

\section{Missing proofs}
\label{sec:missingproof}
We define \(w_i=\sum_{i'=1}^n k(x_i,x_{i'})\, \hat\mu_{\mathcal{X}}(x_{i'})\), and \(k(x_i,x_{i'})\) can be written as $k(x_i, x_j) = \langle\Phi_i,\Phi_j\rangle_{\mathbb{R}^m}$ based on $G = \Phi^\top \Phi$. This can be considered as the discrete version of decomposition in a Hilbert space \(\setH\):  
    \[
    k(x_i,x_{i'}) = \langle \varphi(x_i), \varphi(x_{i'}) \rangle_{\setH}
    \]
    for some feature map \(\varphi: \mathcal{X} \to \setH\).
\begin{theorem}\label{thm:eqn} We have the following equivalent formulation of the optimal transport problem:
\begin{equation}
\begin{aligned}
&\inf_{\pi \in \Pi(\hat\mu_\mathcal{X}, \hat\mu_\mathcal{Y})}
    \sum_{i,j,i',j'=1}^n 
    \pi_{ij}\pi_{i'j'} \, k(x_i,x_{i'}) \,\|y_j-y_{j'}\|^2. \\
&\quad=2\inf_{\pi\in\Pi(\hatmu_{\setX},\hatmu_\setY)} \inf_{A\in \mathbb{R}^{m\times d_y}} 
\Biggl\{ \|A\|^2 \\
&+ \sum_{i=1}^n\sum_{j=1}^n 
    \pi_{ij}\Bigl(\|y_j\|^2 w(x_i) 
    - 2 \langle A, \Phi_i y_j^\top \rangle \Bigr)\Biggr\}.
\end{aligned}
\end{equation}
\end{theorem}

\begin{proof}
We decompose the quadratic cost as
\begin{equation}\label{eq:decomp}
\sum_{i,i'=1}^n \sum_{j,j'=1}^n k(x_i,x_{i'}) \, \|y_j-y_{j'}\|^2 \, \pi_{ij}\,\pi_{i'j'}
= I_1 + I_2 - 2I_3,
\end{equation}
where
\begin{align}
I_1 &= \sum_{i,i'=1}^n \sum_{j,j'=1}^n k(x_i,x_{i'})\, \|y_j\|^2 \,  \pi_{ij}\,\pi_{i'j'}, \label{eq:I1}\\
I_2 &= \sum_{i,i'=1}^n \sum_{j,j'=1}^n  k(x_i,x_{i'})\, \|y_{j'}\|^2 \, \pi_{ij}\,\pi_{i'j'}, \label{eq:I2}\\
I_3 &= \sum_{i,i'=1}^n \sum_{j,j'=1}^n k(x_i,x_{i'})\,\langle y_j, y_{j'} \rangle \,  \pi_{ij}\,\pi_{i'j'}. \label{eq:I3}
\end{align}

For \(I_1\), using the definition of \(w_i=\sum_{i'}k(x_i,x_{i'})\,\hat\mu_{\mathcal{X}}(x_{i'})\),
\begin{equation}\label{eq:I1-simplify}
I_1 = \sum_{i=1}^n\sum_{j=1}^n w_i\,\|y_j\|^2\,  \pi_{ij}.
\end{equation}
Similarly,
\begin{equation}\label{eq:I2-simplify}
I_2 = \sum_{i'=1}^n\sum_{j'=1}^n w_{i'}\,\|y_{j'}\|^2\,  \pi_{i'j'}.
\end{equation}

For \(I_3\), using the kernel representation \(k(x_i,x_{i'})=\langle\Phi(x_i),\Phi(x_{i'})\rangle_{\mathbb{R}^m}\), we obtain
\begin{equation}\label{eq:I3-kernel}
I_3 = \sum_{i,i'=1}^n\sum_{j,j'=1}^n
\langle \Phi(x_i),\Phi(x_{i'})\rangle_{\mathbb{R}^m} \,\langle y_j,y_{j'}\rangle \,
\pi_{ij}\,\pi_{i'j'}.
\end{equation}
Let \(m_i=\sum_{j=1}^n \frac{y_j \,\pi_{ij}}{\hat\mu_{\mathcal{X}}(x_i)}\), then
\begin{equation}\label{eq:I3-mean}
I_3 = \Bigl\|\sum_{i=1}^n\Phi(x_i)\times m_i\, \hat\mu_{\mathcal{X}}(x_i)\Bigr\|^2_{\mathbb{R}^m\times \mathbb{R}^{d_y}}.
\end{equation}

Applying the identity
\[
-\,\|z\|^2 = \inf_{A\in \mathbb{R}^m\times \mathbb{R}^{d_y}} \Bigl\{\|A\|^2_{\mathbb{R}^m\times \mathbb{R}^{d_y}} - 2\langle A,z\rangle_{\mathbb{R}^m\times \mathbb{R}^{d_y}}\Bigr\},
\]
to \eqref{eq:I3-mean}, we obtain
\begin{equation}\label{eq:I3-inf}
-I_3 = \inf_{A\in \mathbb{R}^m\times \mathbb{R}^{d_y}}\Bigl\{\|A\|^2_{\mathbb{R}^m\times \mathbb{R}^{d_y}}
-2\!\!\sum_{i=1}^n\sum_{j=1}^n \langle A,\Phi(x_i)\times y_j\rangle_{\mathbb{R}^m\times \mathbb{R}^{d_y}}\, \pi_{ij}\Bigr\}.
\end{equation}

Combining \eqref{eq:I1-simplify}, \eqref{eq:I2-simplify}, and \eqref{eq:I3-inf}, the overall problem admits the equivalent formulation
\[
\inf_{\pi\in\Pi(\hat\mu_{\mathcal{X}},\hat\mu_{\mathcal{Y}})}\;\inf_{A\in \mathbb{R}^m\times\mathbb{R}^{d_y}}
2\Bigl\{\|A\|^2_{\mathbb{R}^m\times\mathbb{R}^{d_y}}
+\!\!\sum_{i=1}^n\sum_{j=1}^n\Bigl(\|y_j\|^2\,w_i
-2\langle A,\Phi(x_i)\times y_j\rangle_{\mathbb{R}^m\times\mathbb{R}^{d_y}}\Bigr)\pi_{ij}\Bigr\},
\].
\end{proof}

\begin{proposition}
\label{thm:con}
Algorithm~\ref{alg:opt} produces a non-increasing sequence of OT objectives $\{\mathcal{L}_{\mathrm{OT}}^{(t)}\}$ that converges to a stationary point, with per-iteration complexity $\mathcal{O}(n^2)$.
\end{proposition}
\begin{proof}
We interpret Algorithm~\ref{alg:opt} as minimizing the entropy-regularized objective:
\[
\mathcal{L}_{\mathrm{OT}}(\pi,A)
=\|A\|_F^2+\langle C(A),\pi\rangle
+ \varepsilon\sum_{i,j}\bigl(\pi_{ij}\log\pi_{ij}-\pi_{ij}\bigr),
\]
where \(C(A)=w\,Y_{\mathrm{norm}}^\top-2\,\Phi A Y^\top\).
At each iteration, the algorithm performs:
(i) an exact update \(A^{(t+1)}=\arg\min_A \mathcal{L}_{\mathrm{OT}}(\pi^{(t)},A)\) (closed form),
and
(ii) an exact update \(\pi^{(t+1)}=\arg\min_{\pi\in\Pi} \mathcal{L}_{\mathrm{OT}}(\pi,A^{(t+1)})\) (Sinkhorn run to convergence).

\paragraph{Non-increasing.}Exact minimization in each block implies
\(\mathcal{L}_{\mathrm{OT}}(\pi^{(t)},A^{(t+1)})\le \mathcal{L}_{\mathrm{OT}}(\pi^{(t)},A^{(t)})\)
and
\(\mathcal{L}_{\mathrm{OT}}(\pi^{(t+1)},A^{(t+1)})\le \mathcal{L}_{\mathrm{OT}}(\pi^{(t)},A^{(t+1)})\).
Hence \(\{\mathcal{L}_{\mathrm{OT}}^{(t)}\}\) is non-increasing.

\paragraph{Lower boundedness.}
Fix $\pi\in\Pi$. The $A$-subproblem is
\[
\min_A\;\; \|A\|_F^2-2\!\left\langle A,\Phi^\top \pi Y\right\rangle
\;+\;\underbrace{\left\langle w\,Y_{\mathrm{norm}}^\top,\pi\right\rangle
+ \varepsilon\!\sum_{i,j}\!\bigl(\pi_{ij}\log\pi_{ij}-\pi_{ij}\bigr)}_{\text{independent of }A}.
\]
Its gradient and Hessian are
\[
\nabla_A = 2A-2\Phi^\top \pi Y,\qquad 
\nabla_A^2 = 2I \succ 0,
\]
So the subproblem is strictly convex and has a unique minimizer
\[
A^\star(\pi)=\Phi^\top \pi Y.
\]
Plugging this back gives the reduced objective.
\[
\begin{aligned}
\widetilde{\mathcal L}_{\mathrm{OT}}(\pi)
&:= -\bigl\|\Phi^\top \pi Y\bigr\|_F^2
+\big\langle w\,Y_{\mathrm{norm}}^\top,\pi\big\rangle
+ \varepsilon\!\sum_{i,j}\!\bigl(\pi_{ij}\log\pi_{ij}-\pi_{ij}\bigr),
\end{aligned}
\]
where we used
$\|A\|_F^2-2\langle A,\Phi^\top\pi Y\rangle
= \|A-\Phi^\top\pi Y\|_F^2-\|\Phi^\top\pi Y\|_F^2$.

Since the domain of $\pi$ is bounded and $\pi\mapsto \Phi^\top \pi Y$ is continuous, $-\|\Phi^\top \pi Y\|_F$ is lower-bounded.
Similarly, $\langle w\,Y_{\mathrm{norm}}^\top,\pi\rangle$ is lower-bounded.
For the entropy term, by Jensen's inequality,
\[
\sum_{i,j}\pi_{ij}\log\pi_{ij}\ \ge\ -\log(n^2),
\]
and $-\sum_{i,j}\pi_{ij}=-1$.
Therefore
\[
\varepsilon\sum_{i,j}\bigl(\pi_{ij}\log\pi_{ij}-\pi_{ij}\bigr)
\ \ge\ -\varepsilon\bigl(2\log n+1\bigr).
\]
As a result, \(\widetilde{\mathcal L}_{\mathrm{OT}}(\pi)\) is lower bounded and consequently, so is \(\mathcal{L}_{\mathrm{OT}}(\pi,A)\).

\paragraph{Stationary points.} The \(A\)-subproblem is a strongly convex and therefore has the unique minimizer \(A^\star(\pi)=\Phi^\top \pi Y\). The \(\pi\)-subproblem is strictly convex because \(\varepsilon>0\) and therefore has a unique minimizer characterized by the KKT system, that is, there exist dual scalings \(u,v\) such that \(\log \pi^\star = -C(A^\star)/\varepsilon + u\mathbf{1}^\top + \mathbf{1}v^\top\) together with the row and column constraints~\citep{NIPS2013_af21d0c9}. Standard results on two-block coordinate descent then imply that any limit point \((\pi^\star,A^\star)\) is block-optimal and satisfies the first-order conditions
\[
\nabla_A \mathcal{L}_{\mathrm{OT}}(\pi^\star,A^\star)=2A^\star-2\Phi^\top \pi^\star Y=0,
\]
as well as the KKT conditions for \(\pi^\star\). Consequently, any limit point is stationary.

\paragraph{Computational complexity.}For the per-iteration complexity, forming \(C(A)=wY_{\mathrm{norm}}^\top-2\,\Phi A Y^\top\) costs \(\mathcal{O}(n^2+nmd_y)\), updating \(A=\Phi^\top \pi Y\) costs \(\mathcal{O}(nmd_y)\), and one Sinkhorn iteration costs \(\mathcal{O}(n^2)\). With \(T\) inner iterations, one outer iteration costs \(\mathcal{O}\bigl((1+T)n^2+nmd_y\bigr)\),  where $T$ is the number of iterations. When $T$ is bounded ($T$ is usually $\sim 15$ in our experiments), the computational complexity is \(\mathcal{O}(n^2)\).
\end{proof}

\begin{proposition}
Standard GWOT is a special case of generalized GWOT with $k(x,x') = -\|x-x'\|^2$ and can also be optimized via Algorithm~\ref{alg:opt}.
\end{proposition}
\begin{proof}
Consider the squared-loss Gromov--Wasserstein (GW) objective with Euclidean
dissimilarities:
\begin{equation}\label{eq:gw-def}
\mathrm{GW}(\pi)
=\sum_{i,i'=1}^n\sum_{j,j'=1}^n
\big(\, \|x_i-x_{i'}\|^2 - \|y_j-y_{j'}\|^2 \,\big)^2 \; \pi_{ij}\,\pi_{i'j'}.
\end{equation}

Expanding the square yields
\begin{equation}\label{eq:gw-expand}
\begin{aligned}
\mathrm{GW}(\pi)
&=\sum_{i,i'=1}^n \|x_i-x_{i'}\|^4\, \hat\mu_{\mathcal{X}}(x_i)\,\hat\mu_{\mathcal{X}}(x_{i'})
+\sum_{j,j'=1}^n \|y_j-y_{j'}\|^4\, \hat\mu_{\mathcal{Y}}(y_j)\,\hat\mu_{\mathcal{Y}}(y_{j'}) \\
&\quad-\,2\!\!\sum_{i,i'=1}^n\sum_{j,j'=1}^n \|x_i-x_{i'}\|^2\|y_j-y_{j'}\|^2\, \pi_{ij}\,\pi_{i'j'}.
\end{aligned}
\end{equation}

The first two terms in \eqref{eq:gw-expand} depend only on the marginals and 
are therefore independent of the coupling~$\pi$. They can be discarded without 
affecting the minimizer. Thus the minimization of \(\mathrm{GW}(\pi)\) is 
equivalent to
\begin{equation}\label{eq:gw-reduced}
\arg\min_{\pi}\,\mathrm{GW}(\pi)
=\arg\min_{\pi}\Bigl\{-2\!\!\sum_{i,i'=1}^n\sum_{j,j'=1}^n \|x_i-x_{i'}\|^2\|y_j-y_{j'}\|^2\, \pi_{ij}\,\pi_{i'j'}\Bigr\}.
\end{equation}

Define the kernel weight
\begin{equation}\label{eq:kernel-def}
k'(x_i,x_{i'}) := \|x_i-x_{i'}\|^2.
\end{equation}
Then \eqref{eq:gw-reduced} can be written as
\begin{equation}\label{eq:gw-to-ot}
\arg\min_{\pi}\,
-\sum_{i,i'=1}^n\sum_{j,j'=1}^n
k'(x_i,x_{i'}) \,\|y_j-y_{j'}\|^2 \,  \pi_{ij}\,\pi_{i'j'},
\end{equation}
which conforms with our objective~\eqref{eqn:discreteobjective}.

This naturally leads to a variant of Theorem~\ref{thm:eqn}:

\begin{equation}
\begin{aligned}
&\frac{1}{2}\inf_{\pi \in \Pi(\hat\mu_{\mathcal{X}}, \hat\mu_{\mathcal{Y}})} 
    -\sum_{i,i'=1}^n\sum_{j,j'=1}^n  k'(x_i,x_{i'}) \, \|y_j-y_{j'}\|^2 \, \pi_{ij}\,\pi_{i'j'} \\
&\quad=\inf_{\pi\in\Pi(\hat\mu_{\mathcal{X}},\hat\mu_{\mathcal{Y}})}\;\inf_{A\in \mathbb{R}^m\times\mathbb{R}^{d_y}}
\Bigl\{\|A\|^2_{\mathbb{R}^m\times\mathbb{R}^{d_y}} -\!\!\sum_{i=1}^n\sum_{j=1}^n\|y_j\|^2\,w_i\,\pi_{ij}
+2\sum_{i=1}^n\sum_{j=1}^n\langle A,\Phi(x_i)\times y_j\rangle_{\mathbb{R}^m\times\mathbb{R}^{d_y}}\,\pi_{ij}\Bigr\}.
\end{aligned}
\end{equation}
\noindent\textit{where} $K'\in\mathbb{R}^{n\times n}$ is the Gram matrix with $K'_{ii'}=k'(x_i,x_{i'})$ admitting a factorization $K'=\Phi\Phi^\top$ for some $\Phi\in\mathbb{R}^{n\times m}$ whose $i$-th column equals $\Phi(x_i)^\top$, and $w_i:=\sum_{i'=1}^n k'(x_i,x_{i'})\,\hat\mu_{\mathcal X}(x_{i'})$.

The  remaining steps of the algorithm proceed exactly as before: update $\pi$ via the (entropically regularized) Sinkhorn step,
and update $A$ in closed form by $A \leftarrow \Phi^\top \pi Y$.
The monotone decrease argument and the $\mathcal{O}(n^2)$ per-iteration
Complexity carries over verbatim.

\end{proof}

\begin{theorem}\label{thm:dec}
Let $u^*$ denote the drift field corresponding to the optimal solution of~\eqref{eqn:dcfmloss}. Then $u^*$ satisfies
$u^*(z,c,t,0)=\mathbb E[\dot a_t x(0)+\dot b_t x(1) \mid x(t)=z,\,y(1)=c]$, and $u^*(c,z,t,1)=\mathbb E[\dot a_t y(0)+\dot b_t y(1) \mid y(t)=z,\, x(1)=c]$.
\end{theorem}
\begin{proof}
We work at the population level. Let $r\in\{0,1\}$ denote the role indicator ($r=0$ for the $x$-direction, $r=1$ for the $y$-direction), and define, for each training draw, the tuple.
\[
(Z,C,V,t)=
\begin{cases}
(x(t),\ y(1),\ \dot a_t x(0)+\dot b_t x(1),\ t), & r=0,\\[2mm]
(y(t),\ x(1),\ \dot a_t y(0)+\dot b_t y(1),\ t), & r=1,
\end{cases}
\]
where $(x(1),y(1))\sim\pi^{\mathrm{OT}}$, $x(0)\sim p_0^x$, $y(0)\sim p_0^y$, and $x(t)=a_t x(0)+b_t x(1)$, $y(t)=a_t y(0)+b_t y(1)$ with differentiable $a_t,b_t$.

The DCFM population loss is
\[
\mathcal L(u)\;=\;\mathbb E\big[\ \|\,u(Z,C,t,r)-V\,\|^2\ \big],
\]
where the expectation is taken over the joint law of $(Z,C,r,V,t)$. Set the $\sigma$-algebra
\[
\mathcal A:=\sigma(Z,C,r,t).
\]
The random vector $U:=u(Z,C,t,r)$ is thus $\mathcal A$-measurable and square-integrable. By the orthogonal projection identity,
\begin{equation}\label{eq:pyth}
\mathbb E\big[\|V-U\|^2\big]
=\mathbb E\big[\|V-\mathbb E(V\mid \mathcal A)\|^2\big]
+\mathbb E\big[\|U-\mathbb E(V\mid \mathcal A)\|^2\big].
\end{equation}
The rightmost term is minimized when $U=\mathbb E(V\mid \mathcal A)$, hence
\[
u^*(Z,C,t,r)\;=\;\mathbb E\!\left[V\mid Z,C,r,t\right].
\]
Thus, by specifying the conditional expectation pointwise, we obtain the role-wise form
\begin{equation}\label{eq:rolewise}
u^*(z,c,t,r)\;=\;\mathbb E\!\left[V\mid Z=z,\ C=c,\ r,\ t\right].
\end{equation}
Unpacking the definitions of $(Z,C,V)$ under $r=0$ and $r=1$ gives exactly the two cases stated in the theorem:
\[
u^*(z,c,t,0)=\mathbb E[\dot a_t x(0)+\dot b_t x(1) \mid x(t)=z,\,y(1)=c,\,t],\qquad
u^*(c,z,t,1)=\mathbb E[\dot a_t y(0)+\dot b_t y(1) \mid y(t)=z,\,x(1)=c,\,t],
\]
and the explicit dependence on $t$ can be suppressed in notation when no ambiguity arises.

Finally, taking total expectation in \eqref{eq:pyth} and conditioning on the events $\{r=0\}$ and $\{r=1\}$ shows the loss decomposes additively across directions:
\begin{align*}
\mathcal L(u)-\mathcal L(u^*)
&=\;\mathbb E\big[\|u(Z,C,t,r)-u^*(Z,C,t,r)\|^2\big] \\[1ex]
&=\;\mathbb E\big[\|u(Z,C,t,0)-u^*(Z,C,t,0)\|^2 \mid r=0\big]\;\mathbb P(r=0) \\[1ex]
&\quad+\;\mathbb E\big[\|u(Z,C,t,1)-u^*(Z,C,t,1)\|^2 \mid r=1\big]\;\mathbb P(r=1).
\end{align*}
\end{proof}

\section{Evaluation Metrics}
\label{sec:evaluationmetric}
\paragraph{Notation.}
For each evaluation run, we subsample $N$ samples $\{x_i\}_{i=1}^N$ from the testing set, and generate their embedding vectors $\{y_i\}_{i=1}^N$. Then we reconstruct $\{\hat{x}_i\}_{i=1}^N$ using the trained $u_\theta$.
Unless stated otherwise, metrics are reported as mean $\pm$ standard deviation aggregated over $R$ independent runs.

\paragraph{Standardization.}
We apply standardization procedures before computing image-based evaluation metrics:
(i) if needed, single-channel images are broadcast to 3 channels;
(ii) we restore dataset statistics using per-dataset $(\mu,\sigma)$ and clamp to $[0,1]$ for FID;
(iii) for LPIPS, we map images to $[-1,1]$ after restoration.
Distances in pixel space use Euclidean distance on flattened images.

\subsection{Fr\'echet Inception Distance (FID)}
Let $\psi(\cdot)\in\mathbb{R}^{2048}$ denote Inception-V3 pool features applied to images after standardization.
Let $\{\psi(x_i)\}_{i=1}^N$ and $\{\psi(\hat{x}_i)\}_{i=1}^N$ have empirical means and covariances $(\mu_r,\Sigma_r)$ and $(\mu_g,\Sigma_g)$, respectively.
FID is the Fr\'echet distance between Gaussian approximations to these feature distributions:
\begin{equation}
\mathrm{FID} \;=\; \|\mu_r-\mu_g\|_2^2 \;+\; \mathrm{Tr}\!\Big(\Sigma_r + \Sigma_g - 2(\Sigma_r^{1/2}\Sigma_g\Sigma_r^{1/2})^{1/2}\Big).
\end{equation}
The FID score is a widely adopted metric to evaluate the quality of generated samples.


\subsection{Wasserstein distance of embeddings}
Since the objective of the dimension reduction algorithm is to match $\hat{\mu}_\setX$ to a tractable distribution $\hat{\mu}_{\setY}$, it is also important to measure the distance between $\hat{\mu}_{\setY}$ and the target Gaussian distribution. We use the Wasserstein OT distance to characterize such a distributional distance. Let $Y=\{y_j\}_{j=1}^n$ be the embeddings generated by DCFM. We draw the same number of points $Z=\{z_j\}_{j=1}^n$ from the standard Gaussian $\mathcal N(0,I_2)$. Treating both sets as uniform discrete measures, we compute the 2-Wasserstein OT loss between $Y$ and $Z$ with squared Euclidean cost. In practice, we use a Sinkhorn solver with a small entropic regularization: \[
\text{WOT}_\varepsilon(Y,Z)
=\min_{\gamma\in\mathbb{R}_+^{n\times n}}
\sum_{i,j}\gamma_{ij}\,\|y_i-z_j\|^2
+\varepsilon\sum_{i,j}\gamma_{ij}(\log\gamma_{ij}-1)
\quad\text{s.t.}\quad
\sum_j\gamma_{ij}=\tfrac{1}{n},\;
\sum_i\gamma_{ij}=\tfrac{1}{n}.
\]
The reported value is the resulting OT objective, which quantifies how close the inferred embeddings are to the Gaussian reference. Values of FID and OT are reported in Table~\ref{Table:OTandFid} in the main paper.

\subsection{Generalized GWOT objective}
The original generalized GWOT OT loss is
\begin{equation}
\label{eq:gen-ot}
\mathcal L(\pi)
=\sum_{i,i',j,j'=1}^n \pi_{ij}\,\pi_{i'j'}\, k(x_i,x_{i'})\,\|y_j-y_{j'}\|^2.
\end{equation}

In practice, the dimension reduction algorithms would generate a set of discrete samples $Y$ based on $X$. Therefore, we define a binary transport plan that pairs $x_i$ with its inferred $y_i$, i.e.,
\[
\pi_{ij}=\tfrac{1}{n}\,\mathbf 1\{i\text{ matches }j\}.
\]
Substituting this into \eqref{eq:gen-ot} gives
\[
{\text{GWOT}}
= \sum_{i,i',j,j'} \tfrac{1}{n}\mathbf 1\{i\text{ matches }j\}\cdot \tfrac{1}{n}\mathbf 1\{i'\text{ matches }j'\}\, k(x_i,x_{i'})\,\|y_j-y_{j'}\|^2=\frac{1}{n^2}\sum_{i,j=1}^n k(x_i,x_j)\,\|y_i-y_j\|^2.
\]

This objective measures how well the matching between the original samples $x$ and the embeddings $y$ align with user-defined objective.

\subsection{Aggregation Across Runs}
Let $m^{(r)}$ be a scalar metric measured on run $r \in \{1,\dots,R\}$.
We report
\begin{equation}
\overline{m} \;=\; \frac{1}{R}\sum_{r=1}^{R} m^{(r)},
\qquad
s(m) \;=\; \sqrt{\frac{1}{R-1}\sum_{r=1}^{R}\big(m^{(r)}-\overline{m}\big)^2}.
\end{equation}

In all tables, the best results are \textbf{bolded}, while the second-best results are \ul{underlined}.

\subsection{OT Loss (Generalized GWOT)}
The main paper reports FID and OT loss (Wasserstein). Here we report OT loss (Generalized GWOT). 

\begin{table*}[htbp]
\centering
\begin{tabular}{ccccc}
   \toprule
    & MNIST & CIFAR-10 & AFHQ & TinyImagenet \\
   \midrule
   KPCA & 23.5 $\pm$ 0.1 & \ul{42} $\pm$ 1 & 266 $\pm$ 2& 366 $\pm$ 3 \\
   VAE & \ul{15.8} $\pm$ 0.2 & 56 $\pm$ 2 & \ul{67} $\pm$ 1 & \ul{60} $\pm$ 9 \\
   DiffAE & 370 $\pm$ 1 & 383 $\pm$ 5 & 398 $\pm$ 14& 389 $\pm$ 11 \\
   Info-Diffusion & 71 $\pm$ 2 & 298 $\pm$ 4 & 305 $\pm$ 2 & 298 $\pm$ 3\\
   \midrule
   \textbf{Coupled Flow Matching} & \textbf{1.393 $\pm$ 0.001} & \textbf{1.40 $\pm$ 0.02} & \textbf{1.34 $\pm$ 0.03} & \textbf{1.82 $\pm$ 0.03}\\
   \bottomrule
\end{tabular}
\caption{GWOT}
\label{Table:GWOT}
\end{table*}

\section{Additional experiments}
In this section, we provide additional experiment results and visualizations.
\label{sec:additionalexperiment}
\subsection{Generalized GWOT on CIFAR-10, AFHQ, and TinyImageNet}
We report the performance of Algorithm~\ref{alg:opt-adaptive} (generalized GWOT solver) on CIFAR-10, AFHQ, and TinyImageNet, shown in Figure~\ref{fig:GWOTs}.
\begin{figure*}[h]
    \centering
    \makebox[\textwidth][c]{%
        \begin{subfigure}[t]{0.4\textwidth}
            \centering
            \includegraphics[height=3.5cm]{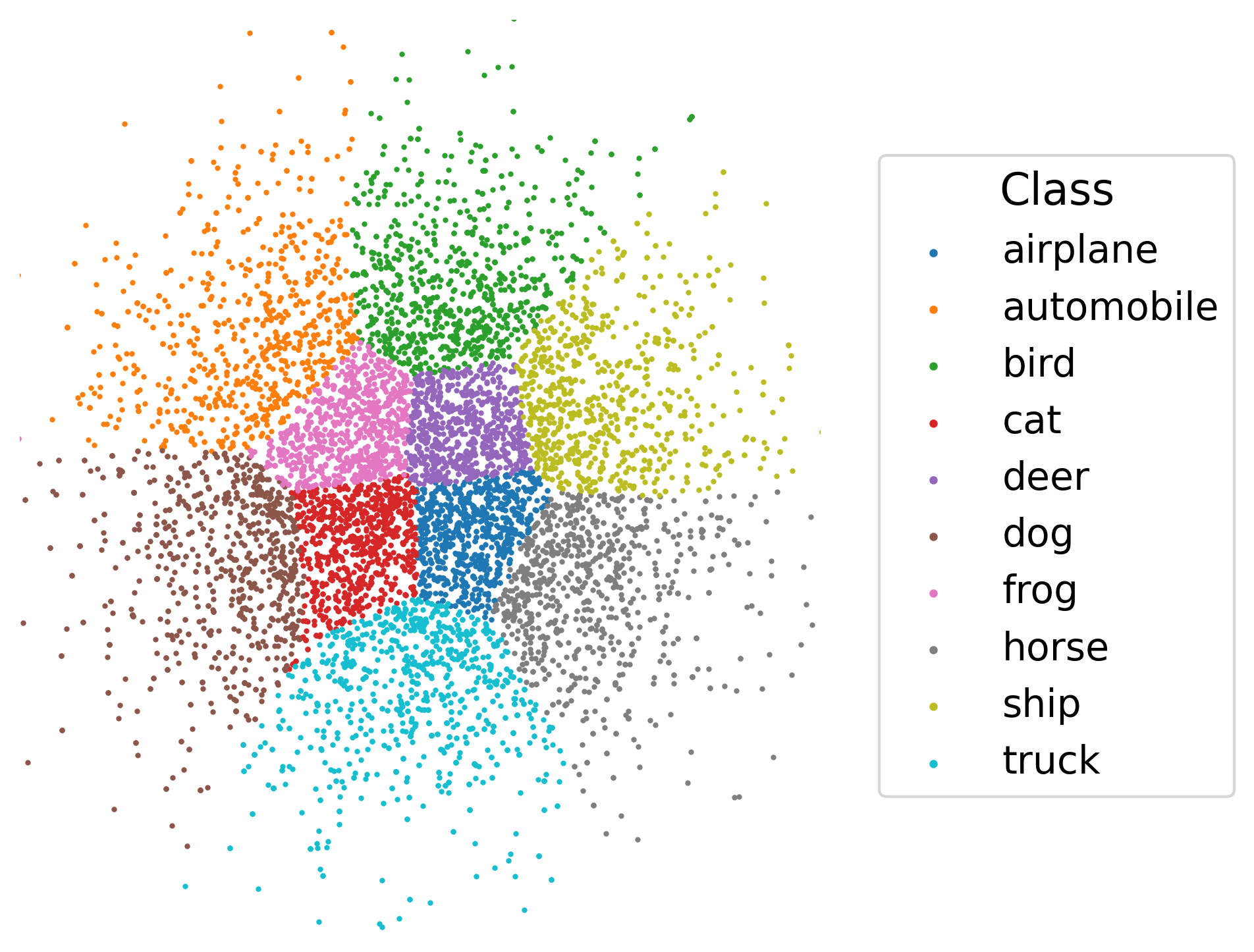}
            \caption{}
        \end{subfigure}
        \hfill
        \begin{subfigure}[t]{0.3\textwidth}
            \centering
            \includegraphics[height=3.5cm]{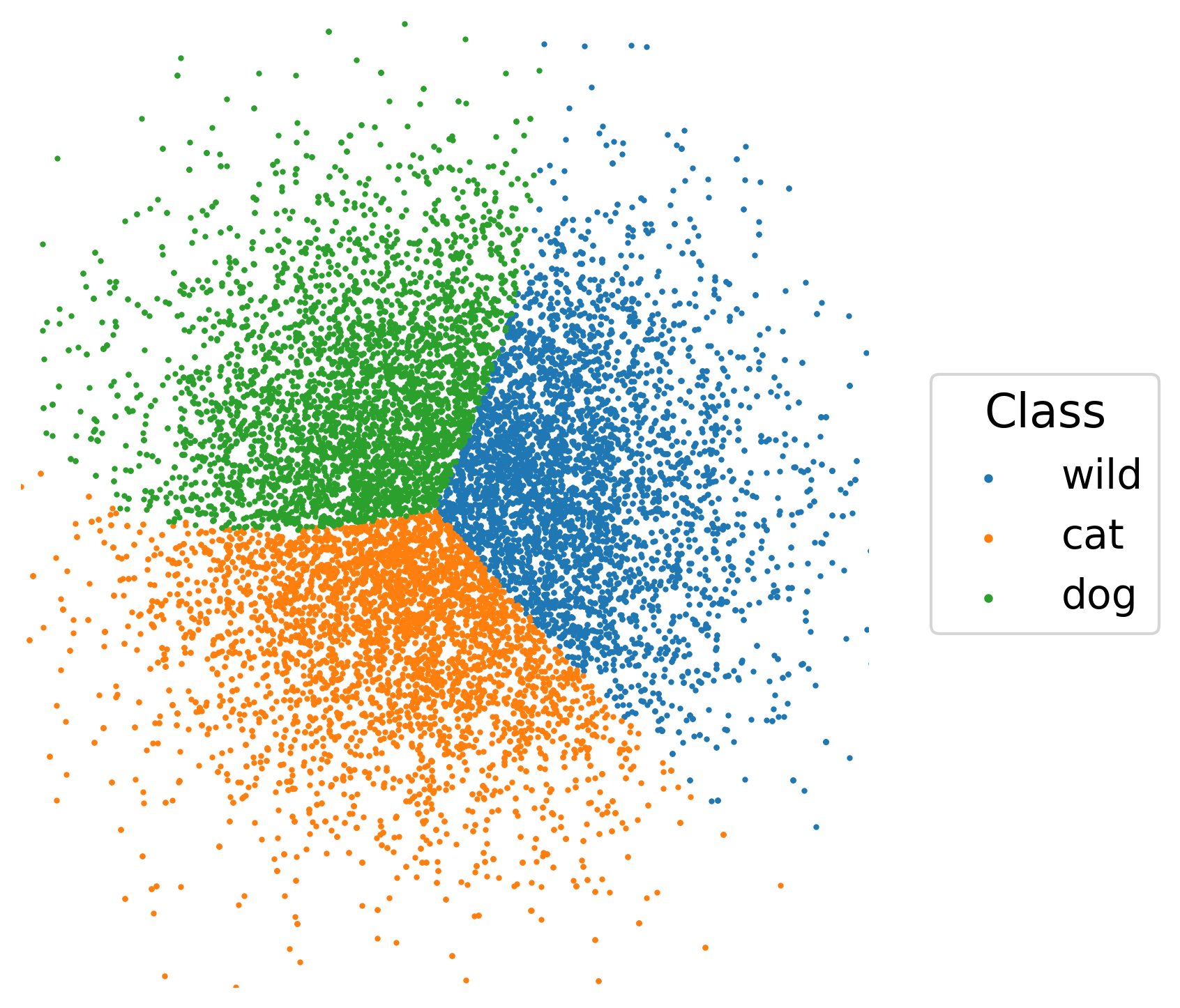}
            \caption{}
        \end{subfigure}
        \hfill
        \begin{subfigure}[t]{0.3\textwidth}
            \centering
            \includegraphics[height=3.5cm]{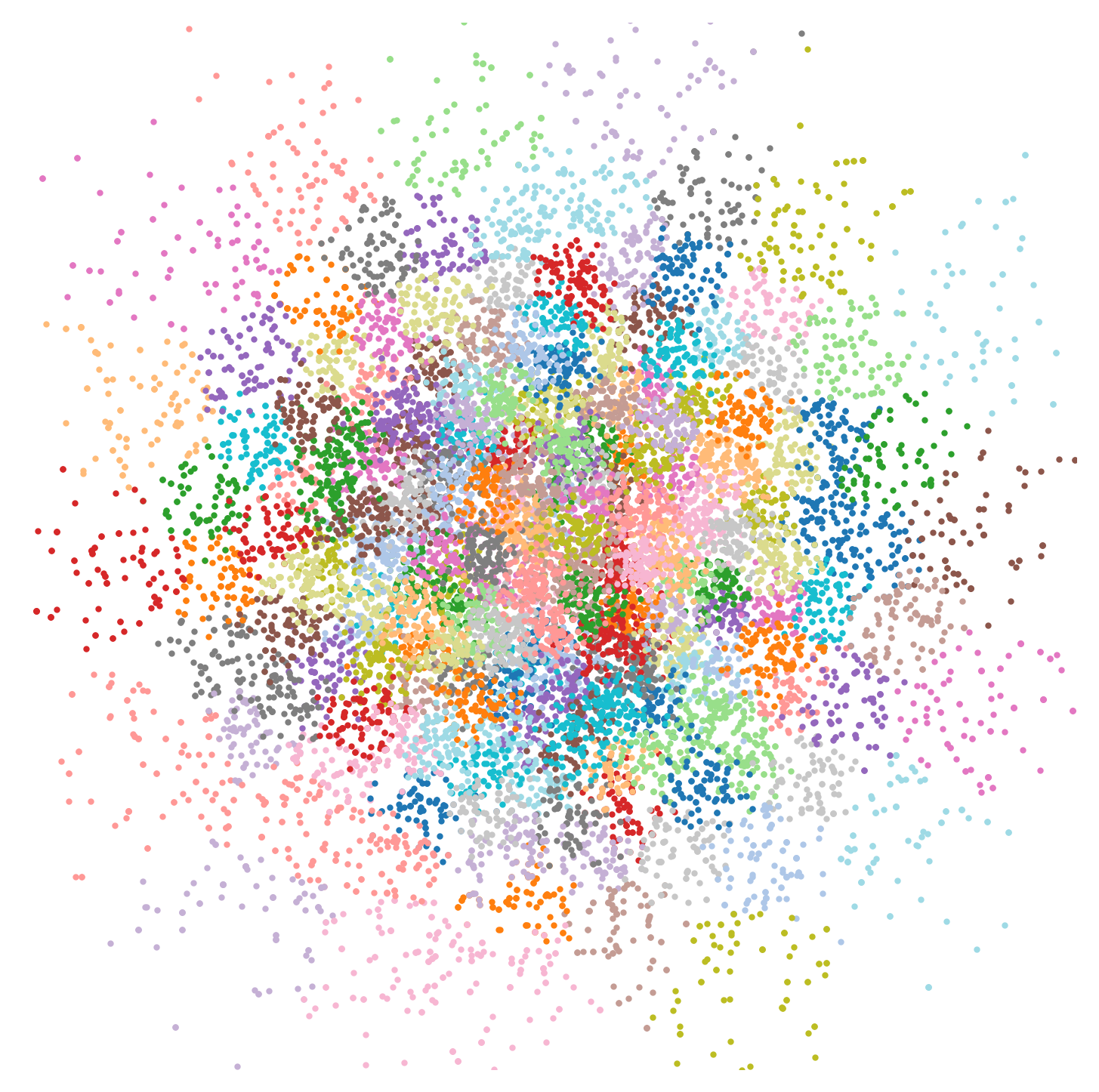}
            \caption{}
        \end{subfigure}
    }
    \caption{Based on the generalized GWOT transport plan $\pi$, each source sample $x$ is associated with a probability distribution over candidate embeddings ${y}$. A single embedding $y$ is randomly sampled according to its assigned weight $\pi_i$. The latent embedding distributions $\mu_{\setY}$ are Gaussian. The datasets used here are (a) CIFAR-10 (10 classes), (b)AFHQ (3 classes), and (c)TinyImageNet (200 classes).}
    \label{fig:GWOTs}
\end{figure*}

In Figure~\ref{fig:GWOTs}, embeddings from the same labels are clearly clustered, even for TinyImagenet with 200 classes. This highlights the effectiveness of Algorithm~\ref{alg:opt-adaptive} to impose user-defined structures. 

\subsection{Extrapolation}

To test whether CPFM internalizes the semantic structure generated by generalized GWOT from Stage~1, we sample a uniform $10\times10$ grid on the 2D square $[-1.5,\,1.5]^2$ in $\setY$. Let $K=10$ and $\Delta=3/(K-1)=1/3$, define $s_u=-1.5+u\Delta$ for $u\in\{0,\dots,9\}$, and the grid points $y_{u,v}=(s_u,s_v)$ for $u,v\in\{0,\dots,9\}$. We treat each $y_{u,v}$ as the embedding and generate an image with \emph{DCFM condition on y}, yielding samples $\{x_{u,v}\}$. We arrange $\{x_{u,v}\}$ in the same $10\times10$ layout (Figure~\ref{fig:geninap}). The generated images show clear class boundaries. Within each class, small moves on the $10\times10$ grid lead to smooth and local visual changes. Distant grid locations yield images that differ more. This behavior is consistent with the heat-kernel prior used in the generalized GWOT stage. Similar samples in $\mathcal X$ are encouraged to occupy nearby positions in the embedding, and dissimilar samples are encouraged to lie farther apart.

\begin{figure*}[h]
    \centering
    \begin{subfigure}[t]{0.3\textwidth}
        \centering
        \includegraphics[width=\textwidth]{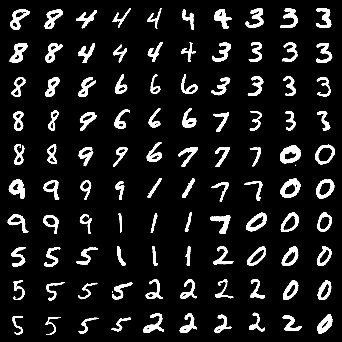}
        \caption{}
    \end{subfigure}
    \begin{subfigure}[t]{0.3\textwidth}
        \centering
        \includegraphics[width=\textwidth]{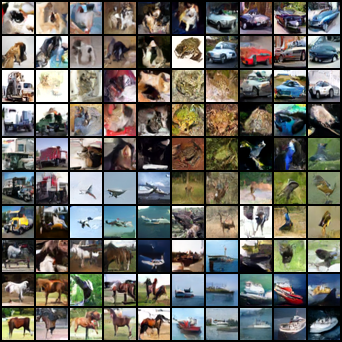}
        \caption{}
    \end{subfigure}
    \begin{subfigure}[t]{0.3\textwidth}
        \centering
        \includegraphics[width=\textwidth]{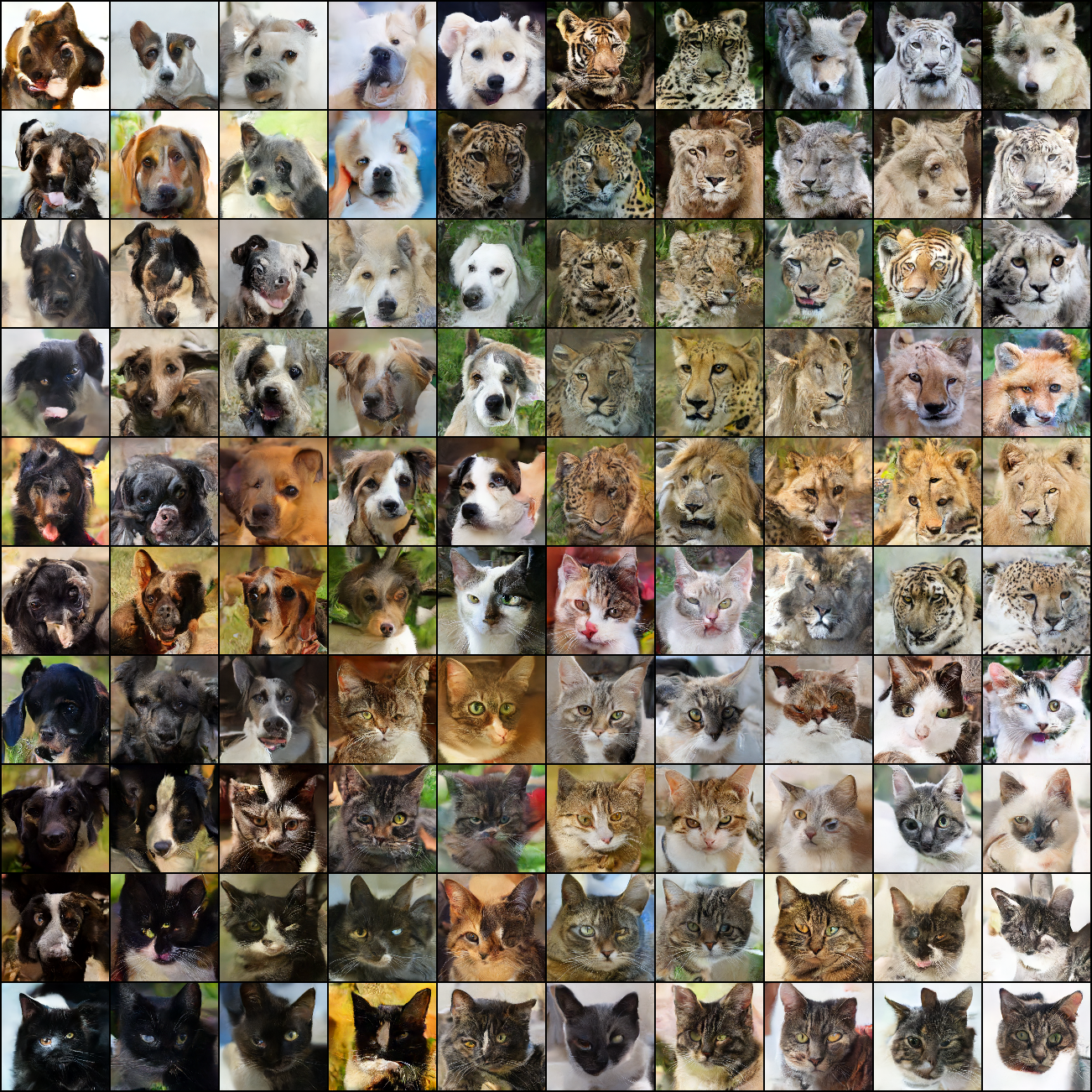}
        \caption{}
    \end{subfigure}

    \caption{Image generations from a uniformly sampled 2D latent grid in $[-1.5, 1.5]^2$. Shown are (a) MNIST, (b) CIFAR-10, and (c) AFHQ.}
    \label{fig:geninap}
\end{figure*}

\subsection{Architecture of the Model}
Our drift network $u_\theta$ uses a U\mbox{-}Net backbone with time conditioning, shown in Figure~\ref{fig:Unet}.
Based on the standard $t$ conditioned U\mbox{-}Net, we treat the two-dimensional embedding $y\in\mathbb R^2$ as a condition as well.
We concatenate $[t, y]$ and inject it at every resolution.
At each encoder or decoder block, we broadcast $[t, y]$ over the spatial grid, concatenate it along the channel dimension with the current feature map of $x$, and apply a linear projection before the block convolutions.
We add self-attention blocks within the U\mbox{-}Net at matched encoder and decoder locations.

\begin{figure*}[h]
    \centering
    \makebox[\textwidth][c]{%
        \begin{subfigure}[t]{0.4\textwidth}
            \centering
            \includegraphics[height=3.8cm]{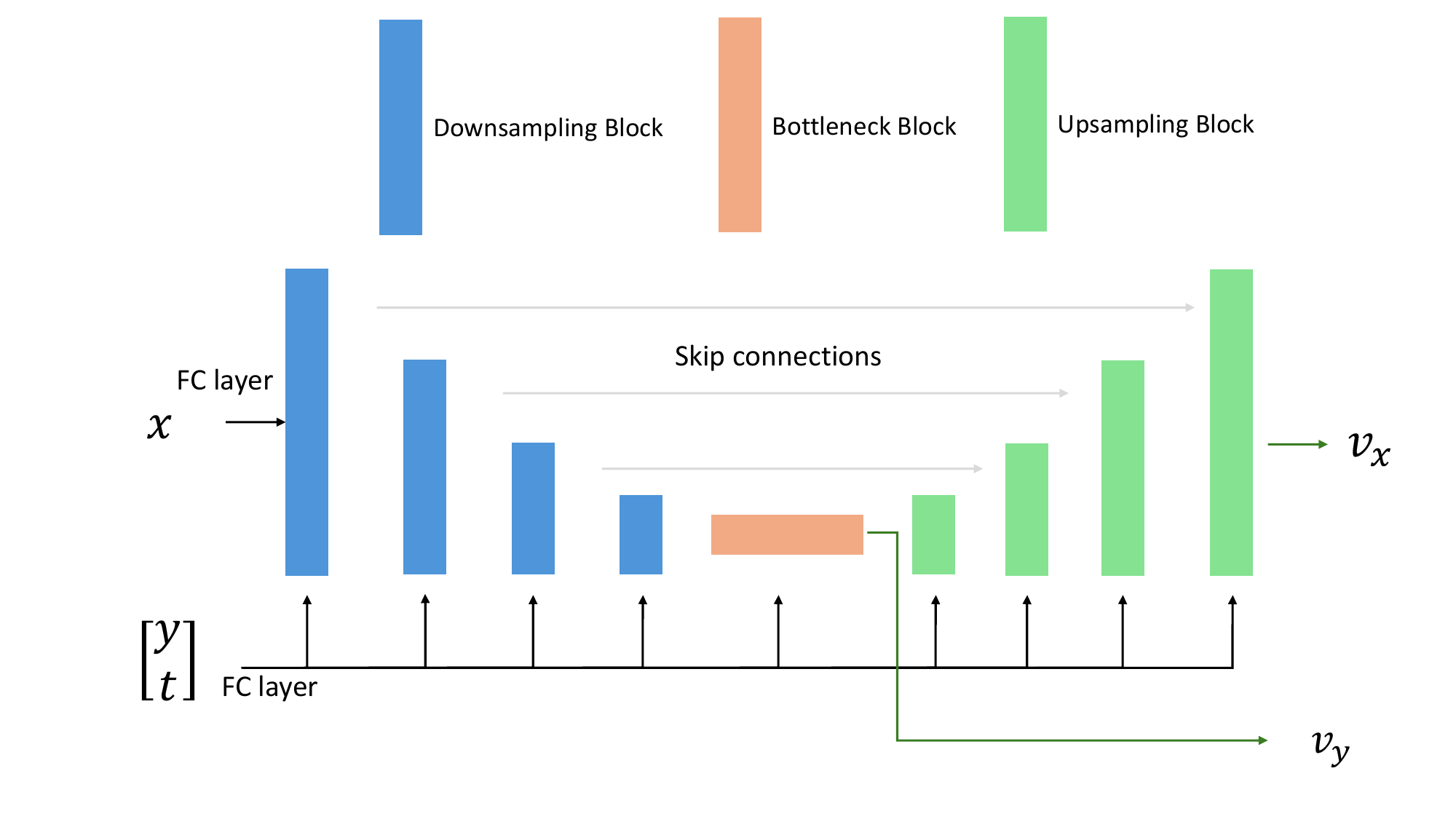}
            \caption{U-Net backbone.}
        \end{subfigure}
        \vspace{-0.8cm}
        \begin{subfigure}[t]{0.3\textwidth}
            \centering
            \includegraphics[height=3.8cm]{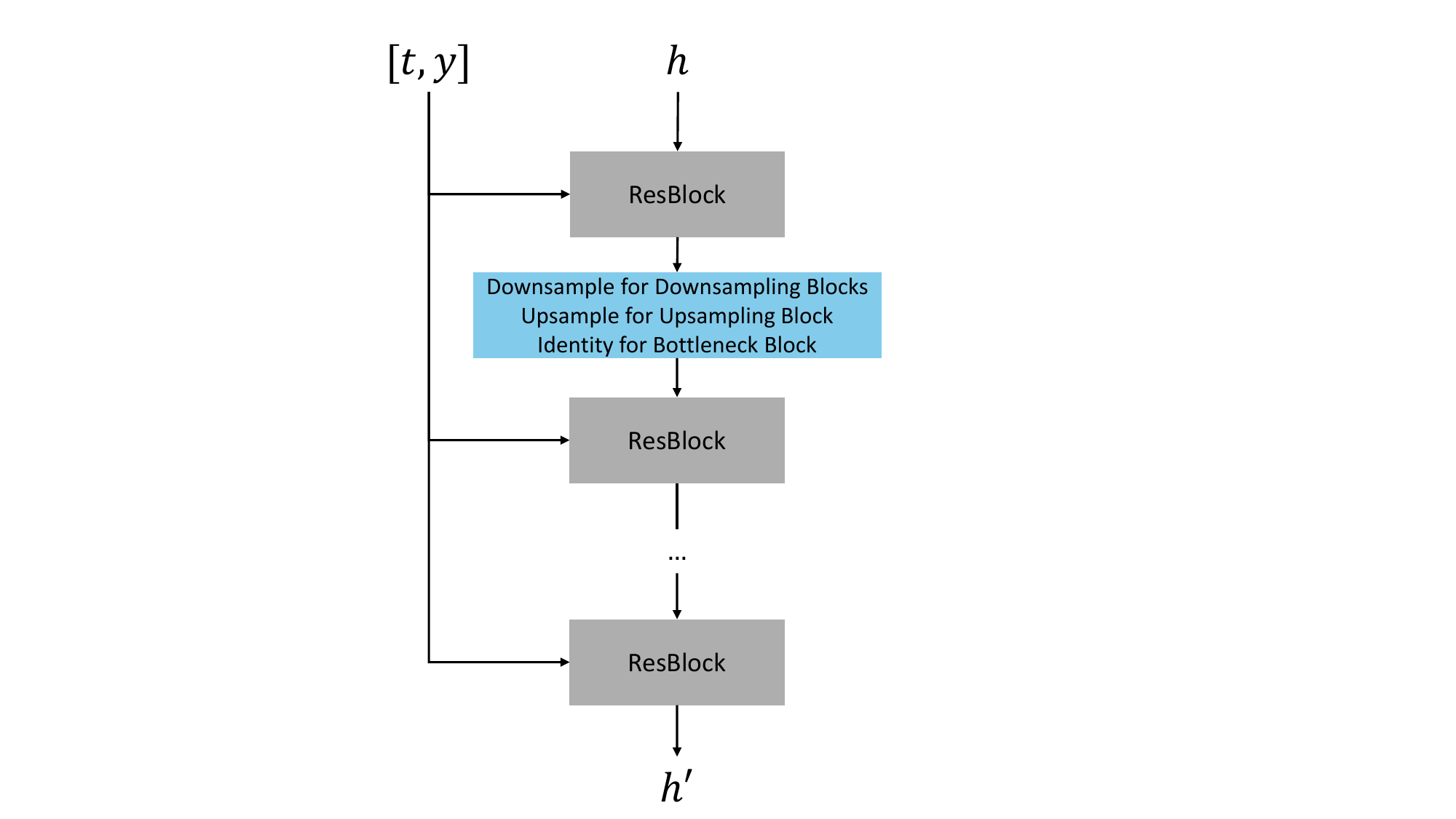}
            \caption{Composition of each block.}
        \end{subfigure}
        \hspace{-1.2cm}
        \begin{subfigure}[t]{0.3\textwidth}
            \centering
            \includegraphics[height=3.8cm]{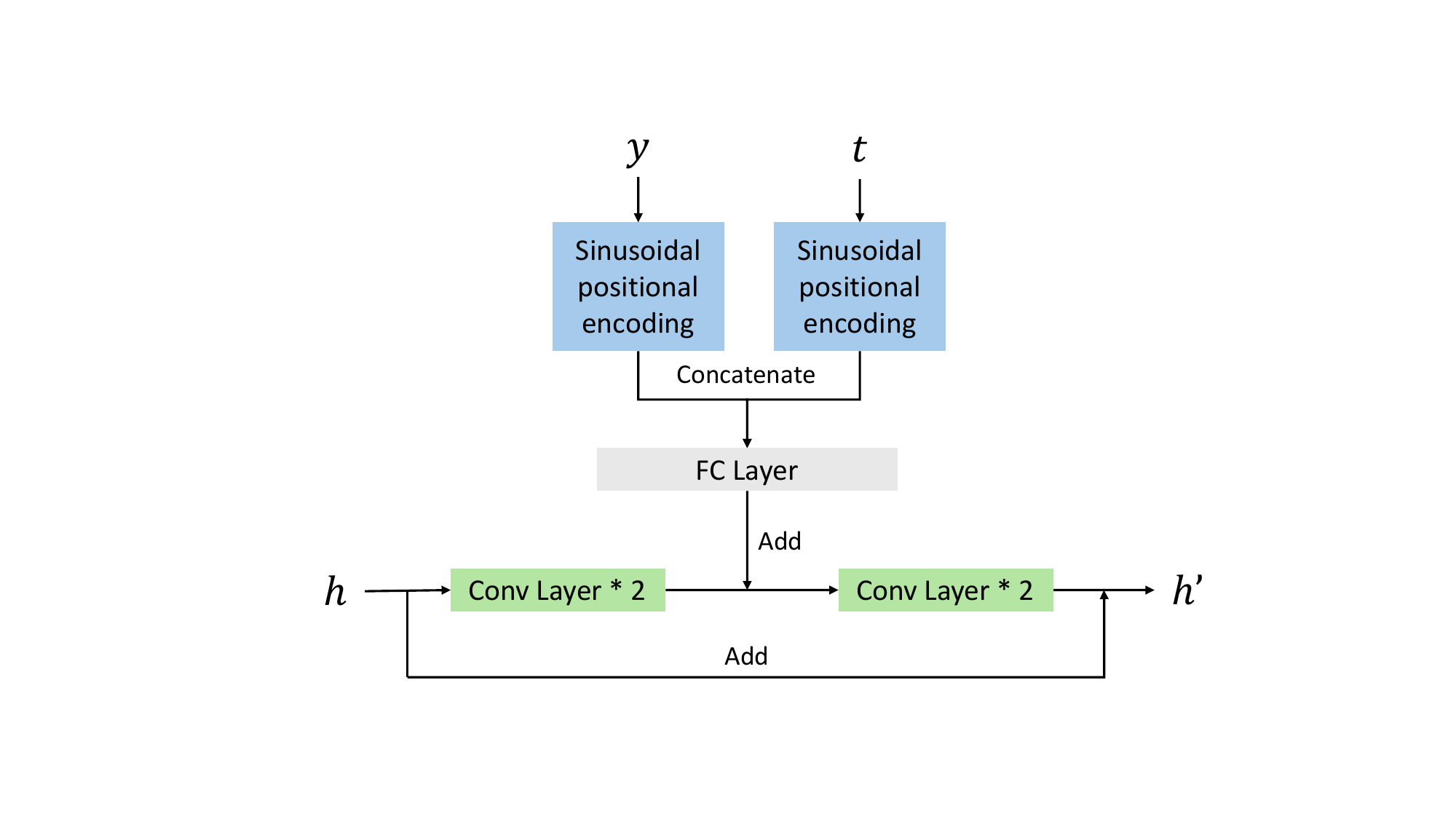}
            \caption{ResBlock.}
        \end{subfigure}
    }
    \caption{Architecture of the Model.}
    \label{fig:Unet}
\end{figure*}

The network has two output heads selected by the role flag $r\in\{0,1\}$.
For the $y$ direction $r=1$, the two-dimensional drift $v_y\in\mathbb R^2$ is produced at the bottleneck head.
For the $x$ direction $r=0$, the drift $v_x$ with the same shape as $x$ is produced by the final decoder head.
During training, only the active head contributes to the loss while the other head is muted.

\subsection{Hyperparameters}
Our model exposes several hyperparameters that control the U\mbox{-}Net capacity and where attention is applied.
\texttt{model\_channels} sets the base width of feature maps.
\texttt{num\_res\_blocks} is the number of residual blocks per resolution level.
\texttt{channel\_mult} is a list of multipliers applied to \texttt{model\_channels} at each level from high to low resolution.
\texttt{attention\_resolutions} lists the spatial sizes at which we insert self-attention blocks in both encoder and decoder.
\texttt{num\_heads} is the number of heads in each attention block.
We use a learning rate of $1\times10^{-4}$ and train all models for 200 epochs.

\begin{table}[h]
\centering
\small
\setlength{\tabcolsep}{5pt}
\begin{tabular}{lccccc}
\toprule
Dataset & \texttt{model\_channels} & \texttt{num\_res\_blocks} & \texttt{channel\_mult} & \texttt{attention\_resolutions} & \texttt{num\_heads} \\
\midrule
MNIST        & 64  & 2 & {[1, 2, 2, 2]}              & {[16]}            & 4 \\
CIFAR\mbox{-}10    & 128 & 2 & {[1, 2, 2, 2]}       & {[16]}         & 4 \\
TinyImageNet & 128 & 2 & {[1, 2, 2, 2]}       & {[16]}         & 4 \\
AFHQ         & 192 & 2 & {[1, 1, 2, 4]}    & {[16, 32]}     & 4 \\
\bottomrule
\end{tabular}
\caption{Default hyperparameter settings used in our runs}
\label{tab:hparams}
\end{table}

\end{document}